\documentclass[twoside]{article}

\usepackage{amsthm}
\usepackage[english]{babel}
\usepackage{csquotes}
\usepackage{scalerel}
\usepackage{microtype}
\usepackage{nccmath}
\usepackage{caption}
\usepackage{multicol}
\usepackage{changepage}   %
\usepackage{enumitem}

\usepackage{courier}

\usepackage{thm-restate}
\usepackage{multirow}
\usepackage{parskip}
\usepackage{fdk-math}
\usepackage[nomarginkeys]{fdk-cite}
\usepackage{geometry}
\usepackage{adjustbox}
\usepackage{booktabs} 
\usepackage{mdframed} 
\usepackage{nicefrac}
\usepackage{xspace}
\usepackage{relsize}
\usepackage{marginnote}
\usepackage{xcolor}

\newtheoremstyle{mytheoremstyle} %
    {}                    %
    {}                    %
    {\itshape}                   %
    {}                           %
    {\scshape}                   %
    {.}                          %
    {.5em}                       %
    {}                           %

\theoremstyle{mytheoremstyle}
\newtheorem{theorem}{Theorem}

\mdfdefinestyle{mdframedthmbox}{%
    leftmargin=.0\textwidth,
    rightmargin=.0\textwidth,
    innertopmargin=0.75em,
    innerleftmargin=.5em,
    innerrightmargin=.5em,
}

\newenvironment{thmbox}{%
    \vspace{0.75em}
    \begin{mdframed}[style=mdframedthmbox]%
}{%
    \end{mdframed}%
    \ignorespacesafterend%
}

\let\ts\textstyle
\let\ds\displaystyle

\def\m!{\mkern-2mu}
\def\n!{\mkern-1mu}
\def\p!{\mkern-.5mu}
\def\rm!{\mkern2mu}
\def\rn!{\mkern1mu}

\let\nablasymbol\nabla
\RenewDocumentCommand{\nabla}{e{_^}}{%
    \mathop{}\!%
    \nablasymbol%
    \IfValueT{#1}{_{\mspace{-4mu}#1}}%
    \IfValueT{#2}{^{#2}}%
}

\newcommand{\acro}[1]{{\smaller #1}}
\newcommand{\EM}{\acro{EM}\xspace}
\newcommand{\aGD}{\acro{GD}\xspace}
\newcommand{\aKL}{\acro{KL}\xspace}
\newcommand{\aNLL}{\acro{NLL}\xspace}
\newcommand{\aM}{\acro{M}\xspace}
\newcommand{\aE}{\acro{E}\xspace}
\newcommand{\aMLE}{\acro{MLE}\xspace}
\newcommand{\aMAP}{\acro{MAP}\xspace}
\newcommand{\aEM}{\acro{EM}\xspace}
\newcommand{\aFIM}{\acro{FIM}\xspace}
\newcommand{\aECM}{\acro{ECM}\xspace}

\newcommand{\Estep}{\aE-step\xspace}
\newcommand{\Mstep}{\aM-step\xspace}

\newcommand{\data}{x}
\newcommand{\lat}{z}
\newcommand{\iter}{t}
\newcommand{\iterstart}{1}

\newcommand{\maxiter}{T}
\newcommand{\fullsum}{\sum_{\iter=\iterstart}^\maxiter}
\newcommand{\fullmin}{\min_{\iter\leq\maxiter}}
\newcommand{\np}{\theta}
\newcommand{\nptm}{\np_{\iter-1}}
\newcommand{\npt}{\np_\iter}
\newcommand{\nptt}{\np_{\iter+1}}

\newcommand{\phit}{\phi_\iter}
\newcommand{\phitt}{\phi_{\iter+1}}
\newcommand{\npT}{\np_{\maxiter}}
\newcommand{\npStart}{\np_{\iterstart}}

\renewcommand{\mp}{\mu}

\newcommand{\mpt}{\mp_\iter}
\newcommand{\mptt}{\mp_{\iter+1}}

\newcommand{\A}{\m!A}
\newcommand{\AS}{\m!A^{\n!*}\n!}
\newcommand{\Breg}[1]{D_{#1}}
\newcommand{\D}{\Breg{\n!\A\n!}}
\newcommand{\DS}{\Breg{\n!\AS\n!}}
\newcommand{\stats}{S}
\newcommand{\statso}{s}
\newcommand{\s}{s}

\newcommand{\thetat}{\npt}
\newcommand{\thetatt}{\nptt}

\newcommand{\Aug}{\Loss^{\m!+}}

\NewDocumentCommand{\Q}{mg}{Q(\IfNoValueTF{#2}{\cdot}{#2} \cond #1)}
\NewDocumentCommand{\Hent}{mg}{H(\IfNoValueTF{#2}{\cdot}{#2}m\cond #1)}
\RenewDocumentCommand{\Q}{mg}{Q_{\p!#1\p!}(\IfNoValueTF{#2}{\cdot}{#2})}
\RenewDocumentCommand{\Hent}{mg}{H_{\p!#1\p!}(\IfNoValueTF{#2}{\cdot}{#2})}
\NewDocumentCommand{\Qprime}{mg}{Q'_{\p!#1\p!}(\IfNoValueTF{#2}{\cdot}{#2})}
\NewDocumentCommand{\Qt}{g}{\Q{\npt\n!}{#1}}
\NewDocumentCommand{\Ht}{g}{\Hent{\npt\n!}{#1}}
\newcommand{\QtStar}{Q_{\p!\npt\n!\p!}^{*}}

\newcommand{\const}{\text{const}}

\DeclareMathOperator{\jac}{J\!}
\DeclareMathOperator{\Jac}{J\!}

\ExplSyntaxOn
\NewDocumentCommand{\expectation}{mm}{%
    \int \tl_if_empty:nTF{#1}{}{#1\,\,} #2 \dif{z}
}
\ExplSyntaxOff

\newcommand{\appendixtitle}[1]{
    ~\vspace{-1em}
    \hsize\textwidth\linewidth\hsize\toptitlebar%
    \vspace{-.6em}{\centering\Large\bfseries #1 \par}%
    \bottomtitlebar%
    \vskip 0.2in plus 1fil minus 0.1in%
}

\NewDocumentCommand{\bigOT}{g}{\mathop{\tilde{\mathcal{O}}}\IfNoValueF{#1}{\nleft(#1\nright)}}

\RenewDocumentCommand{\KL}{smm}{%
    \mathrm{KL}{%
        \IfBooleanTF{#1}{%
            [\rn!#2\rn!\Vert\rn!#3\rn!]%
        }{%
            \nleft[\rn!#2\rn!\middle\Vert\rn!#3\rn!\nright]%
        }%
    }%
}

\newcommand{\someset}{\Theta}

\newcommand{\eqnumberref}[1]{\raisebox{.05em}{$\scriptstyle{(#1)}$}}%
\newcommand{\nablaLoss}{\nabla\m!\Loss}
\renewcommand{\medint}{\raisebox{.1em}{$\n!\ds\scaleobj{0.8}{\int}\n!$}}
\newcommand{\um}[1]{\scalebox{0.75}[1.0]{$#1$}}
\def\mneg{\um{\,-\m!\n!}}

\usepackage{tikz}

\RenewDocumentCommand{\Expect}{som}{%
    \E_{\IfNoValueTF{#2}{}{#2}}%
    \IfBooleanTF{#1}{%
        [#3]%
    }{%
        \nleft[#3\nright]%
    }%
}%

\newcommand{\Dh}{\Breg{h}}
\renewcommand{\dh}{\nabla h}
\newcommand{\dA}{\nabla\A}
\newcommand{\dAS}{\nabla\AS}

\newcommand{\inleq}{\m!=\m!}

\usepackage[accepted]{aistats2021}

\usepackage{array}
\newcolumntype{L}[1]{>{\raggedright\let\newline\\\arraybackslash\hspace{0pt}}m{#1}}
\newcolumntype{C}[1]{>{\centering\let\newline\\\arraybackslash\hspace{0pt}}m{#1}}
\newcolumntype{R}[1]{>{\raggedleft\let\newline\\\arraybackslash\hspace{0pt}}m{#1}}

\newcommand{\peqref}[1]{(\hyperref[#1]{Equation~\ref{#1}})}
\newcommand{\makesection}[1]{\section{\MakeUppercase{#1}}}
\newcommand{\figscale}[1]{\scalebox{1.0}{#1}}

\usepackage{lengthconvert}

\newcommand{\thetitle}{%
Homeomorphic-Invariance of \aEM{}: Non-Asymptotic Convergence\\
in \aKL Divergence for Exponential Families via Mirror Descent}

\begin{document}

\runningtitle{Homeomorphic-Invariant Analysis of EM}
\twocolumn[
    \aistatstitle{\thetitle}%
    \aistatsauthor{ Frederik Kunstner \And Raunak Kumar \And Mark Schmidt}%
    \aistatsaddress{University of British Columbia \And Cornell University \And University of British Columbia\\ 
    Canada \acro{CIFAR} \acro{AI} Chair ({Amii})
    }%
]

\makeatletter 
\renewcommand{\Notice@String}{%
ArXiv version. 
Accepted at AISTATS \@conferenceyear.\\
{\url{proceedings.mlr.press/v130/kunstner21a.html}}}
\makeatother

\begin{abstract}\noindent\vskip-.1em
Expectation maximization (\EM{}) is the default algorithm 
for fitting probabilistic models with missing or latent variables, 
yet we lack a full understanding of its non-asymptotic convergence properties.
Previous works show results along the lines of
``\EM{} converges at least as fast as gradient descent''
by assuming the conditions for the convergence of gradient descent 
apply to \EM{}.
This approach is not only loose, 
in that it does not capture that \EM{} can make more progress than a gradient step, 
but the assumptions fail to hold for textbook examples of \EM{}
like Gaussian mixtures.
In this work we first show that for the common setting of exponential
family distributions,
viewing \EM{} as a mirror descent algorithm leads to convergence 
rates in Kullback-Leibler (\aKL) divergence.
Then, we show how the \aKL divergence 
is related to first-order stationarity via Bregman divergences.
In contrast to previous works, the analysis is invariant to the choice of parametrization and holds with minimal assumptions.
We also show applications of these ideas to local linear (and superlinear) convergence rates, generalized \aEM,  and non-exponential family distributions.
\end{abstract}

\makesection{Introduction}

Expectation maximization (\aEM) is the most common approach 
to fitting probabilistic models with missing data or latent variables. 
\aEM was formalized by \citet{dempster1977maximum}, 
who discussed a wide variety of earlier works that independently discovered the algorithm
and domains where \aEM is used.
They
already listed
multivariate sampling,
normal linear models, 
finite mixtures, 
variance components, 
hyperparameter estimation, 
iteratively reweighted least squares, and factor analysis. 
To this day, \aEM continues to be used for these applications and others, 
like semi-supervised learning \citep{ghahramani1995learning},
hidden Markov models \citep{rabiner1989tutorial},
continuous mixtures %
\citep{caron2008sparse}, 
mixture of experts \citep{jordan1995convergence}, 
image reconstruction \citep{figueiredo2003algorithm}, 
and graphical models \citep{lauritzen1995algorithm}. 
The many applications of \aEM have made the work of \citeauthor{dempster1977maximum}
one of the most influential in the field.

\begin{figure}[t]
	\vspace{-.5em}
	\centering
	\adjustbox{trim={.0\width} {.165\height} {.0\width} {.08\height}, clip}{%
		\includegraphics[width=.8\columnwidth]{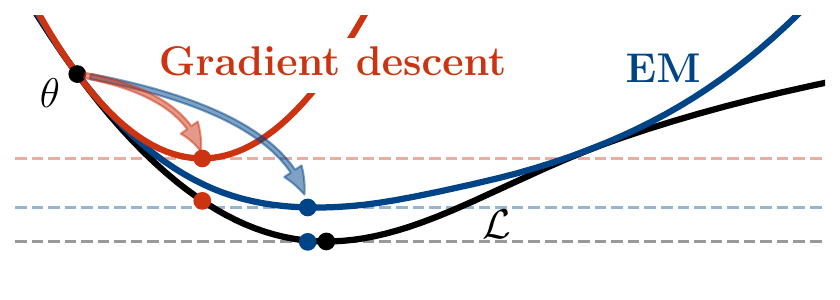}%
	}
	\caption{
		The surrogate optimized by \aEM is a tighter bound 
		on the objective $\Loss$ than the
		 quadratic bound implied by smoothness,
		optimized by gradient descent.
	}
	\vskip-.5em
	\label{fig:double-surrogate}
\end{figure}

Since the development of \aEM 
and subsequent clarifications on the necessary conditions for convergence
\citep{boyles1983convergence,wu1983convergence}, 
a large number of works have shown convergence results for \aEM
and its many extensions, leading to a variety of insights about the algorithm,
such as the effect of the ratio of missing information
\citep{xu1996convergence,ma2000asymptotic}
and the sample size 
\citep{wang2015high,yi2015regularized,daskalakis2016ten,balakrishnan2017statistical}.
However, existing results on the global, non-asymptotic convergence of \aEM\ 
rely on proof techniques developed for gradient descent on smooth functions, 
which rely on quadratic upper-bounds on the objective.\footnotemark
\footnotetext{
As \aEM is a maximization algorithm,
we should say ``gradient ascent'' and ``lower-bound''. 
But we use the language of minimization 
to make connections to ideas from the 
optimization
literature more explicit.
}
Informally, this approach argues that the maximization step of 
the surrogate constructed by \aEM\
does at least as well as gradient descent on a quadratic surrogate 
with a constant step-size, as illustrated in \cref{fig:double-surrogate}.

\begin{figure}[t]
\centering
	\centering
	\adjustbox{height=2.25cm}{%
		\adjustbox{trim={.02\width} {.07\height} {.02\width} {.06\height},clip}{%
			\includegraphics[width=.76\columnwidth]{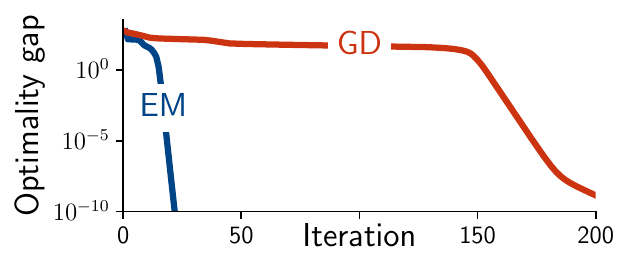}%
		}%
	}
	\caption{
		Performance of \aEM and gradient descent (\aGD)
		with constant step-size, selected by grid-search, for a Gaussian mixture model on the Old Faithful dataset.
		The large gap between the two methods suggests 
		that existing theory for gradient descent 
		is insufficient to explain the performance of \aEM.
	}
	\label{fig:em-vs-gd}
\end{figure}

The use of smoothness as a starting assumption leads to results that imply
that \aEM\ behaves as a gradient method with a constant step-size.
If true, there would be no difference between \aEM\
and its gradient-based variants \citep[e.g.][]{lange2000optimization}.
This does not hold, however, and the resulting convergence rates are inevitably loose;
\aEM\ makes more progress %
than this worst-case bound
even on simple problems, as shown in \cref{fig:em-vs-gd}.

Another issue is that, 
similarly to how Newton's method is invariant to affine reparametrizations,
\aEM is invariant to any homeomorphism \citep{varadhan2004squared};
the steps taken by \aEM are the same for any continuous, invertible reparametrization.
This is not reflected by current analyses
because the parametrization of the problem influences
the smoothness of the function and the resulting convergence rate.
For these reasons, the general frameworks proposed in the optimization literature
\citep{xu2013block,mairal2013optimization,razaviyayn2014successive,paquette2018catalyst}
where \aEM\ is a special case, 
do not reflect that \aEM\ is faster than typical members of these frameworks
and yield loose analyses.

Most importantly, the assumption that the objective function 
is bounded by a quadratic does not hold in general.
Results relying on smoothness do not apply, for example,
to the standard textbook illustration of \aEM:
Gaussian mixtures with learned covariance matrices
\citep{bishop2007prml, murphy2012probabilistic}.
This is shown in \cref{fig:non-smooth-gaussian}.
The smoothness assumption might be a reasonable simplification for local analyses, 
as it only needs to hold over a small subspace of the parameter space.
In this setting, it does not detract from the main contribution of works 
investigating statistical properties or large-sample behavior.
It does not hold, however, for global convergence analyses with arbitrary initializations.
Our focus in this work is analyzing the classic \aEM algorithm when run for a finite number of iterations on a finite dataset, the setting in which people have been using \aEM for over 40 years and continue to use today.

We focus on the application of \aEM to exponential family models, 
of which Gaussian mixtures are a special case. 
Exponential families are by far the most common setting 
and an important special case as the \Mstep has a closed form solution.
Modern stochastic and online extension of \aEM
also rely on the form of exponential families to efficiently 
summarize past data 
\citep{neal1998view,sato1999fastlearning,delyon1999convergence}.

The main tool for the analysis is the 
Kullback-Leibler (\aKL) divergence to describe distances between parameters.
This approach was initially used to derive asymptotic convergence results
\citep{csiszar1984information,chretien2000kullback,tseng2004analysis}
and to describe extensions of \aEM or \aEM-like algorithms
\citep[e.g.][]{banerjee2005clustering,brookes2020viewofem,amid2020divergence}.
But it has not yet been applied to non-asymptotic convergence analyses.
By using the \aKL divergence between the distributions
rather than the Euclidean distance between their parameters,
the results do not rely on invalid smoothness assumptions
and are invariant to the choice of parametrization.

Focusing on convergence to a stationary point, 
as the \aEM objective $\Loss$ is non-convex, 
an informal summary of 
the main difference between previous analyses using smoothness and 
our results is that, after $T$ iterations,

\vspace{.1em}
\begin{tabular}{rllcr}
	{\bf \smaller Smoothness:}      & $\!\!\fullmin \norm{\nablaLoss(\npt)}^2$ 
	& $\!\!\!\!\!\!\leq\!\!\!\!\!\!$ & $L \frac{\Loss(\npStart) - \Loss^*}{\maxiter}$
	\\[.4em]
	{\bf \smaller \aKL divergence:} & $\!\!\fullmin \KL{\nptt}{\npt}$			  
	& $\!\!\!\!\!\!\leq\!\!\!\!\!\!$ & $\hphantom{L} \frac{\Loss(\npStart) - \Loss^*}{\maxiter}$
\end{tabular}
\vspace{.1em}

where
$\Loss^*$ is the optimal value of the objective,
$\Loss(\npStart) - \Loss^*$ is the initial optimality gap
and $L$ is the smoothness constant.
For non-smooth models, such as Gaussians with learned
covariances 
(Fig.~\ref{fig:non-smooth-gaussian}),
$L = \infty$ and the bound is vacuous, 
whereas bounds in \aKL divergence
do not depend on problem-specific constants.
We show how the \aKL divergence relates to stationarity conditions 
for non-degenerate problems in \cref{sec:convergence-results}.

\begin{figure}[t]
	\centering
	\adjustbox{height=2.25cm}{%
			\adjustbox{trim={.01\width} {.005\height} {.03\width} {.02\height},clip}{%
				\includegraphics[width=0.9\columnwidth]{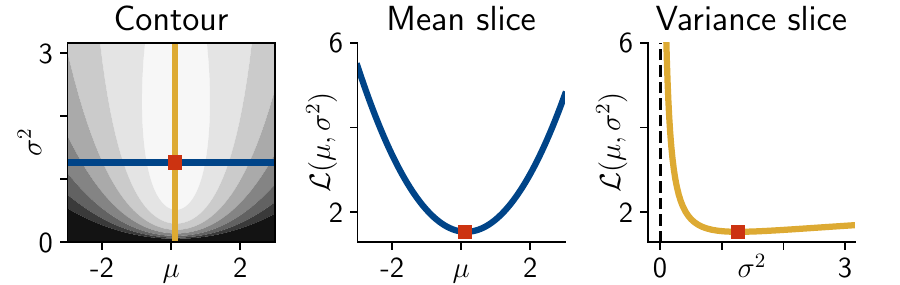}%
			}%
	}
	\caption{
		An exponential family distribution that cannot be smooth;
		fitting a Gaussian $\Normal{\mu, \sigma^2}$, 
		including its variance.
		As the loss diverges to $\infty$, 
		the objective cannot be upper-bounded by a quadratic function.
	}
	\label{fig:non-smooth-gaussian}
\end{figure}

The key observation for exponential families is that \Mstep
iterations match the moments 
of the model to the sufficient statistics of the data.
We show that in this setting, \aEM can be interpreted as a mirror descent update, 
where each iteration minimizes the linearization of the objective 
and a \aKL divergence penalization term 
(rather than the gradient descent update 
which uses the Euclidean distance between parameters instead).
While the connection between \aEM and exponential families is far 
from new, as it predates the codification of \aEM by \citet{dempster1977maximum}
\citep[eg,][]{blight70estimation},
the further connection to mirror descent to describe its behavior is, 
to the best of our knowledge, not acknowledged in the literature.
More closely related to general optimization, 
our work can be seen as an application of the recent perspective of mirror descent 
as defining smoothness relative to a reference function,
as presented by \citet{bauschke2017descent,lu2018relatively}.

Our main results are that we:
\begin{itemize}[leftmargin=1.5em]
	\item Show that \aEM for the exponential family is a mirror descent algorithm, 
	and that the \aEM objective is relatively smooth
	in \aKL divergence.
	\item Show the first homeomorphic-invariant non-asymptotic \aEM convergence rate, and how the \aKL divergence between iterates is related to stationary points and the natural gradient.
	\item Show how the ratio of missing information affects the non-asymptotic linear (or superlinear) convergence rate of \aEM around minimizers.
	\item Extend the results to generalized \aEM, 
	where the \Mstep is only solved approximately. 
	\item Discuss how to handle cases where the \Mstep 
	is not in the exponential family
	(and might be non-differentiable) by analyzing %
	the \Estep.
\end{itemize}

\newcommand{\sectionNameEmAndExpFam}{Expectation-Maximization and Exponential Families}%
\makesection{\sectionNameEmAndExpFam{}}
\label{sec:em-and-expfam}
Before stating our results,
we introduce the \aEM algorithm and necessary background on exponential families.
For completeness, we provide additional details in \cref{app:recap} and 
refer the reader to \citet{wainwright2008graphical} for a full treatment of the subject.

\aEM applies when we want to maximize the likelihood $p(x \cond \np)$ 
of data $x$ given parameters $\np$, 
where the likelihood depends on unobserved variables $z$.
By marginalizing over $z$, we obtain the 
negative log-likelihood (\aNLL), that we want to minimize (to maximize the likelihood),
\alignn{%
	\Loss(\np) 
	= - \log p(\data \cond \np)
	= - \log \medint p(\data, \lat \cond \np) \dif{z},
	\label{eq:objective}
}
where $p(\data, \lat \cond \np)$ is the complete-data likelihood.
The integral here is multi-dimensional if $z$ is, 
and a summation for discrete values, 
but we write all cases as a single integral for simplicity.
\aEM is most useful when the complete-data \aNLL, $-\log p(\data, \lat \cond \np)$,
is a convex function of $\np$
and solvable in closed form if $z$ were known.
\aEM defines the surrogate $\Q{\np}{\phi}$,
which estimates $\Loss(\phi)$ using the expected values 
for the latent variables at $\np$,
\aligns{
	\Q{\np}{\phi} = - \medint \log p(\data, \lat \cond \phi) \, p(\lat \cond \data, \np) \dif{\lat},
}
and iteratively updates $\np_{t+1} \in \argmin_\phi \Q{\np_t}{\phi}$.
The computation of the surrogate $\Q{\np}$ and its minimization 
are typically referred to as the \Estep and \Mstep.

A useful decomposition of the surrogate, 
shown by \citet{dempster1977maximum}, is the equality
\alignn{\label{eq:em-as-loss-and-entropy}\begin{aligned}
	&\Q{\np}{\phi} = \Loss(\phi) + \Hent{\np}{\phi},
	\\
	&\text{where }
	\Hent{\np}{\phi} = - \medint \log p(\lat \cond \data, \phi) \, p(\lat \cond \data, \np) \dif{\lat}
\end{aligned}}
is an entropy-like term minimized at $\phi=\np$. That is,
\aligns{%
	\Hent{\np}{\np} \leq \Hent{\np}{\phi}
	&&\text{and}&&
	\nabla \Hent{\np}{\np} = 0.
}
This gives two fundamental results about \aEM.
Up to a constant, the surrogate is an upper-bound on the objective 
and improvement on $Q_{\np}$ translates to improvement on $\Loss$,
and the gradients of the loss and the surrogate
match at the point it is formed, 
$\nabla \Q{\np}{\np} = \nablaLoss(\np)$.

\subsection{Exponential families}

\begin{figure}[t]%
	\centering%
		\adjustbox{trim={.0\width} {.075\height} {.0\width} {.075\height},clip}{%
			\includegraphics[width=.99\columnwidth]{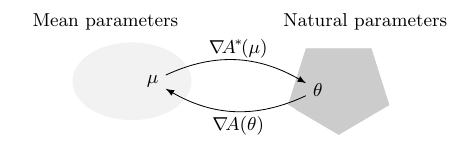}%
		}%
	\caption{
		The gradient of the log-partition function and its dual, $(\nabla A, \nabla \AS)$,
		form a bijection between the natural and mean parameters $\np, \mu$.
	}
	\label{fig:duality-draft}
\end{figure}

Many canonical applications of \aEM,
including mixture of Gaussians, are special cases 
where the complete-data distribution, given a value for the latent variable $z$, 
is an exponential family distribution;
\alignn{\label{eq:primer-exp-fam}\begin{aligned}
	p(\data, \lat \cond \np) &\propto 
	\exp\paren{\lin{\stats(\data, \lat), \np} - \A(\np)},
\end{aligned}}
where
$\stats$, $\np$, and $\A$ are the 
sufficient statistics, 
natural parameters, 
and 
log-partition function
of the distribution.
Exponential family models are an important special case 
as the \Mstep has a closed form solution, 
and the update depends on the data only through the sufficient statistics.
The solution for the maximum likelihood estimate (\aMLE) given $x$ and $z$
can be found from the stationary point of the complete log-likelihood, 
\alignn{\label{eq:stationary-point}
	\nabla \log p(\data, \lat \cond \np) = 
	\stats(\data, \lat) - \nabla \A(\np) = 0.
}
The gradient $\nabla \A$
yields the expected sufficient statistics,
$\nabla \A(\np) = 
\mathbb{E}\,
\raisebox{-.1em}{$\scriptstyle p(\data, \lat \cond \np)$} 
[\stats(\data, \lat)]$,
also called mean parameters and denoted by $\mu$.
The log-partition function defines
a bijection between 
the natural and mean parameters.
Its inverse is given by $\nabla \AS$, 
the gradient of the convex conjugate of $\A$, %
$
	\AS(\mu) = \sup_\np \braces{\lin{\np, \mu} - \A(\np)},
$
such that $\mp = \nabla \A(\np)$ and $\np = \nabla \AS(\mp)$,
as illustrated in \cref{fig:duality-draft}.
The solution to \cref{eq:stationary-point}, given by 
\aligns{
	\nabla \A(\np) = \stats(\data, \lat) && \implies &&
	\np = \nabla \AS(\stats(\data, \lat)),
}
is called moment matching %
as its finds the parameter that, 
in expectation, generates the observed statistics.

To connect \aEM\ and mirror descent, we use the Bregman divergence induced 
by a convex function $h$; the difference between the function and its linearization,
\alignn{
	\label{eq:primer-bregman-def}
	\Breg{h}(\phi, \np) = h(\phi) - h(\np) - \lin{\nabla h(\np), \phi - \np}.
}
For exponential families, the Bregman divergence induced 
by the log-partition $\A$ is the \aKL divergence
\aligns{
	\D(\phi, \np) = \KL{p(\data, \lat \cond \np)}{p(\data, \lat \cond \phi)}.	
}
The Bregman divergences induced by $\A$ 
and its conjugate $\AS$ have the following relation (note the ordering)
\alignn{\label{eq:dual-bregman}
	\D(\phi, \np) = \DS(\nabla \A(\np), \nabla \A(\phi)).
}
Both expressions are the same \aKL divergence, but differ 
in the parametrization used to express the distributions.

\newcommand{\sectionNameEmAndMD}{EM and Mirror Descent}
\makesection{\sectionNameEmAndMD{}}
\label{sec:em-and-mirror-descent}

Although \aEM iterations strictly decrease in the objective function
if such decrease is possible locally, this does not directly imply convergence
to stationary points, even asymptotically \citep{boyles1983convergence,wu1983convergence}.
The progress at each step could decrease faster than the objective.
Characterizing the progress to ensure convergence requires additional assumptions.

Local analyses typically assume that the \aEM 
update contracts 
the distance to a 
local minima $\np^*\!$,
\aligns{
	\norm{\nptt - \np^*} \leq c \norm{\npt - \np^*},
}
for some $c < 1$.
On the other hand, global analyses typically assume the surrogate is \emph{smooth},
meaning that
\aligns{
	\norm{\nabla \Q{\boldsymbol\cdot}{\np} - \nabla \Q{\boldsymbol\cdot}{\phi}}
	\leq 
	L \norm{\np - \phi}.
}
for all $\theta$ and $\phi$,
and some fixed constant $L$. This is equivalent
to assuming the following upper bound holds,
\aligns{
	\Loss(\phi) \leq \Loss(\np) + \lin{\nablaLoss(\np), \phi - \np} 
	+ \frac{L}{2}\norm{\np - \phi}^2.
}
While the local analyses assumptions are reasonable, 
the worst-case value of $L$ for global results can be infinite, 
as in the simple example of \cref{fig:non-smooth-gaussian}.
Instead, we show that the following upper-bound
in \aKL divergence
holds without additional assumptions.

\begin{restatable}{proposition}{restateEmAsMd}
\label{prop:equivalence}
For exponential family distributions, 
the \Mstep update in Expectation-Maximization 
is equivalent to the minimization of the following upper-bound;
\alignn{\label{eq:mirror-descent-upper-bound}
	\Loss(\phi)
	\leq
	\Loss(\np) 
	+ \lin{\nablaLoss(\np), \phi - \np}
	+ \D(\phi, \np),
}
where $\A$ is the log-partition of the complete-data distribution,
and $\D(\phi,\np) = \KL{p(\data,\lat\cond\np)}{p(\data,\lat\cond\phi)}$.
\end{restatable}

While the upper bound is still expressed 
in a specific parametrization to describe the distributions, 
the \aKL divergence is a property of the distributions,
independent of their representation.
As this upper-bound is the one minimized by the \Mstep, 
it is a direct description of the algorithm rather than 
an additional surrogate used for convenience, 
as was illustrated in \cref{fig:double-surrogate}.

This gives an interpretation of \aEM in terms of the mirror descent algorithm 
\citep{nemirovski1983,beck2003mirror},
the minimization of a first-order Taylor expansion 
and Bregman divergence as in \cref{eq:mirror-descent-upper-bound},
with step-size $\alpha \n!=\n! 1$.
In the recent perspective of mirror descent framed as relative smoothness
\citep{bauschke2017descent,lu2018relatively},
the objective function is 1-smooth relative to $\A$.
Existing results \citep[e.g.][Theorem 3.1]{lu2018relatively}
then directly imply the following local result, 
up to non-degeneracy assumptions \ref{ass:1}--\ref{ass:3}
discussed in the next section.
\begin{restatable}{corollary}{restateCorConvex}
\label{cor:convex}
For exponential families, %
 if \aEM is initialized in a locally-convex region with minimum $\np^*\!$, 
\alignn{
	\Loss(\npT) - \Loss(\np^*)
	\leq 
	\frac{1}{\maxiter} \, \KL{p(\data,\lat\cond\npStart)}{p(\data,\lat\cond\np^*)}.
}
\end{restatable}
\vspace{-.5em}
This is the first non-asymptotic convergence rate for \aEM 
that does not depend on problem-specific constants.
\vspace{-2em}
\begin{proof}[Proof of \cref{prop:equivalence}]
Recall the decomposition of the surrogate
in terms of the objective and entropy term,
$\Q{\np}{\phi} = \Loss(\phi) + \Hent{\np}{\phi}$
in \cref{eq:em-as-loss-and-entropy}. It gives
\aligns{%
	\Loss(\phi) - \Loss(\np)
	= \Q{\np}{\phi} - \Q{\np}{\np} + 
	\Hent{\np}{\np} - \Hent{\np}{\phi},
}
where $\Hent{\np}{\np} - \Hent{\np}{\phi} \leq 0$
as $\Hent{\np}{\phi}$ is minimized at $\phi=\np$.
We will show that for exponential families, 
\aligns{%
	\Q{\np}{\phi} - \Q{\np}{\np}
	=
	\lin{\nablaLoss(\np), \phi - \np} + \D(\phi, \np),
}
which implies the upper-bound in 
\cref{eq:mirror-descent-upper-bound}
and that its minima matches that of $\Q{\np}{\phi}$.

If the complete-data distribution is in the exponential family,
the surrogate in natural parameters is 
\alignn{\label{eq:def-q-expfam}\begin{aligned}
	\Q{\np}{\phi}
	&=
	- \medint\log p(\data, \lat \cond \phi) \, p(\lat \cond \data, \np) \dif{z},
	\\[-.5em]
	&=
	- \medint\brackets{\lin{\stats(\data, \lat), \phi} - \A(\phi) } \, p(\lat \cond \data, \np) \dif{z},
	\\ 
	&=
	- \lin{\textstyle\Expect[p(\lat \cond \data, \np)]{\stats(\data, \lat)}, \phi} + \A(\phi).
\end{aligned}}
Using $\s(\np) = \Expect[p(\lat \cond \data, \np)]{\stats(\data, \lat)}$ for the expected sufficient statistics\footnote{The sufficient statistics $\s(\np)$ also depend on $\data$. We do not write $\s(\np, \data)$ as $\data$ is fixed and the same at each iteration.}
and expanding 
$\Q{\np}{\phi} - \Q{\np}{\np}$ yields
\aligns{
	\Q{\np}{\phi} - \Q{\np}{\np}
	&=%
	- \lin{\s(\np), \phi - \np}
	+ \A(\phi) - \A(\np),
	\\
	&\stackrel{{(1)}}{=} 
	- \lin{s(\np) - \nabla \A(\np), \phi - \np}
	+ \D(\phi, \np),
	\\
	&\stackrel{{(2)}}{=} 
	\lin{\nablaLoss(\np), \phi - \np}
	+ \D(\phi, \np),
}
where \eqnumberref{1} adds and subtracts 
$\lin{\nabla \A(\np), \phi - \np}$
to complete the Bregman divergence 
and \eqnumberref{2} uses that the gradient of the surrogate 
and the objective match at $\np$, 
\aligns{%
	\nablaLoss(\np) = \nabla \Q{\np}{\np} = \nabla \n! \A(\np) - \s(\np).
	\tag*{\qedhere}
}
\end{proof}
\vspace{-.5em}
This perspective extends to 
stochastic approximation \citep{robbins1951stochastic}
variants of \aEM, which are becoming increasingly relevant 
as they scale to large datasets.
Algorithms such as incremental, stochastic and online \aEM
\citep{neal1998view,sato1999fastlearning,cappe2009online}
average the observed sufficient statistics to update the parameters.
This can be cast as stochastic mirror descent \citep{nemirovski2009robust}
with step-sizes decreasing as $\nicefrac{1}{t}$.
For brevity, we leave the derivation to \cref{app:proof-equivalence-em-md}.

\newcommand{\sectionNameAssumptions}{Assumptions and Open Constraints}
\makesection{\sectionNameAssumptions{}}
\label{sec:assumptions}
\newcommand{\shortspace}{\vspace{.3em}}
Before diving into convergence results, 
we discuss the assumptions needed for the method 
to be well defined.
\begin{enumerate}[label=\textbf{A\arabic*}]
	\item \label{ass:1}
	\newcommand{\assumptionOne}{%
	The complete-data distribution $p(\data, \lat \cond \np)$ 
	is a steep,
	minimal exponential family distribution. %
	}
	The complete-data distribution $p(\data, \lat \cond \np)$ 
	is a steep,$\!$\footnotemark{}$\mkern.15mu$ minimal exponential family distribution.%
	\footnotetext{%
		The family is steep if its log-partition function satisfies
		$\lim_{i\to\infty}\norm{\nabla A(\theta_i)} \to \infty$ for any sequence 
		$\theta_1, \theta_2, ... \in \Omega = \text{int}(\text{dom}(A))$ 
		converging to a boundary point of $\Omega$.
	}%
\end{enumerate}
\ref{ass:1} implies the continuity and differentiability of $\Loss$, 
that the surrogate has a unique solution,
and that the natural and mean parameters are well defined.
It is the statistical equivalent to the assumption in the mirror descent literature 
that $\A$ essentially smooth,
which implies that the mappings $\nabla \A, \nabla \AS$ are well-defined.
\ref{ass:1} is satisfied in the most commons applications of \aEM in machine learning,
including Gaussian mixtures.

The next assumptions deal with 
a further subtle issue that arises when we attempt to apply results 
from the optimization literature to \aEM,
like the generic frameworks of 
\citet{xu2013block}, \citet{mairal2013optimization} or \citet{razaviyayn2014successive}.
The parameters of the distributions optimized by \aEM are typically 
constrained to a subset $\np \in \Omega$,
like that probabilities sum to one 
and that covariance matrices are positive-definite. 
To handle constraints, those analyses assume
access to a projection onto the constraint set $\Omega$. 
However, this does not hold for common settings of \aEM
like mixtures of Gaussians. 
When the boundaries of
the constraint set are open, the projection operator does not exist
(there is no ``closest positive-definite matrix'' to a
matrix that is not positive-definite).
An additional complication 
is related to the existence of a lower-bound on the objective. 
For example, in Gaussian mixtures, we can drive the objective 
to $\um{-}\infty$ by centering a Gaussian on a single data point and
shrinking the variance towards zero.
The existence of such degenerate solutions 
is challenging for non-asymptotic convergence rates,
as results typically depend on the optimality gap 
$\Loss(\np) - \Loss^*$
and are vacuous if it is unbounded.
To avoid those degenerate cases, we make the following assumptions.

\begin{enumerate}[label=\textbf{A\arabic*}]
	\setcounter{enumi}{1}
	\item \label{ass:2}
	\newcommand{\assumptionTwo}{%
	The objective function is lower-bounded by some 
	$\Loss^* > \um{-}\infty$ on the constraint set $\Omega$. 
	}
	\assumptionTwo{}
\end{enumerate}
\vspace{-.5em}
\begin{enumerate}[label=\textbf{A\arabic*}]
	\setcounter{enumi}{2}
	\item \label{ass:3}
	\newcommand{\assumptionThree}{%
	The sub-level sets $\Omega_{\np} = \{
	\phi \in \Omega : \Q{\np}{\phi} \leq \Q{\np}{\np} \}$
	are compact (closed and bounded).
	}
	\assumptionThree{}
\end{enumerate}

One approach to ensure the \aEM updates are well-defined 
is to add regularization, in the form of a proper conjugate prior.
If the parameters approach the boundary (or diverge in an unbounded direction),
the prior acts as a barrier and diverges to $\infty$ rather than $\scalebox{0.75}[1.0]{$-$}\infty$.
The minimum of the surrogate is then finite 
and in 
$\Omega$
at every iteration, 
without the need for projections.
This is illustrated in \cref{fig:barrier}.
For simplicity of presentation, 
we assume \ref{ass:2} and \ref{ass:3} hold
and discuss maximum a posteriori (\aMAP) estimation in \cref{app:priors}.

\begin{figure}[t]
	\centering
	\figscale{%
	\adjustbox{scale={1}{1}}{%
			\null
				\adjustbox{trim={.0\width} {.08\height} {.0\width} {.08\height},clip}{%
					\includegraphics[width=.8\columnwidth]{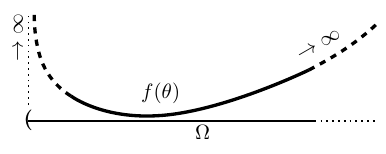}%
				}%
		}%
	}%
	\caption{
		Example of a barrier function with compact sub-level sets on an open set $\Omega$,
		satisfying \ref{ass:2} and \ref{ass:3}.
		Even if $\Omega$ is open, as $f$ goes to $\infty$ at the boundary
		and is convex, the minimum is guaranteed to be in $\Omega$.
	}
	\label{fig:barrier}
\end{figure}

\newcommand{\sectionNameConvergenceEM}{Convergence of EM for Exponential Families}
\makesection{\sectionNameConvergenceEM}
\label{sec:convergence-results}
\label{sec:em-convergence-expfam}

We now give the main results for the convergence 
of \aEM to stationary points for exponential families. This analysis takes advantage of existing tools for the analysis of mirror descent,
but in the less-common non-convex setting.
Detailed proofs are deferred to \cref{app:proof-em-expfam}.

\begin{restatable}{proposition}{convergenceEmAsMd}
\label{thm:convergence-em-as-md}
Under assumptions \ref{ass:1}--\ref{ass:3}, 
\aEM for exponential family distributions converges at the rate 
\aligns{
	\fullmin\,
	\KL{p(\data, \lat \cond \nptt)}{p(\data, \lat \cond \npt)}
	\leq 
	\frac{\Loss(\npStart) - \Loss^*}{T}.
}
\end{restatable}
\vspace{-.5em}
While this result implies the distribution fit by \aEM stops changing, 
it does not---in itself---guarantee progress toward a stationary point
as it is also satisfied by an algorithm that does not move, 
$\nptt \n!= \npt$. 
In the standard setting of gradient descent with constant step-size,
\cref{thm:convergence-em-as-md} 
is the equivalent of the statement that the distance between iterates
$\norm*{\nptt-\npt}$ converges.
As $\norm*{\nptt - \npt} \m!\propto\m! \norm*{\nablaLoss(\npt)}$,
it also implies that the gradient norm converges.
A similar result holds for \aEM,
where measuring distances between iterates with $\D$ 
leads to stationarity in the dual divergence $\DS$.

\clearpage
Recall that the \Mstep finds a stationary point 
of the upper-bound in \cref{eq:mirror-descent-upper-bound}. Setting its derivative to 0,
using $\nabla_\phi\, \D(\phi, \npt) = \nabla \A(\phi) - \nabla \A(\npt)$ yields
\aligns{
	\nablaLoss(\npt) - \nabla \A(\npt) + \nabla \A(\nptt) = 0.
}
Using the expansion of the gradient in terms of the observed statistics $\s(\npt)$
and the mean parametrization $\mpt = \nabla \A(\npt)$, we obtain the moment matching update;
finding the mean parameters $\mptt$ that generate the observed sufficient statistics $\s(\npt)$ in expectation;
\aligns{
	\mptt = \s(\npt) = \mpt - \nablaLoss(\npt).
}
Expressing the \aKL divergence as the dual Bregman divergence 
$\DS(\mptt, \mpt)$ \peqref{eq:dual-bregman} then gives
\aligns{\begin{aligned}
		\KL{\nptt}{\npt}
		&= \DS(\mptt, \mpt)%
		= \DS(\s(\npt), \mpt).
\end{aligned}}
This adds a measure of stationarity to \cref{thm:convergence-em-as-md};
\begin{restatable}{corollary}{convergenceEmAsMdTwo}
\label{cor:convergence-em-as-md}
Under assumptions \ref{ass:1}--\ref{ass:3}, 
\aligns{
	\fullmin\,
	\DS(\s(\npt), \mpt)
	\leq 
	\frac{\Loss(\npStart) - \Loss^*}{\maxiter}.
}
The observed sufficient statistics $s(\npt)$ and 
mean parameters $\mpt$ are the two parts of the gradient, 
$\nablaLoss(\npt) = \mpt - \s(\npt)$,
and $\DS(\s(\npt), \mpt) = 0$ implies $\nabla \Loss(\npt) = 0$.
\end{restatable}
\cref{cor:convergence-em-as-md} is the Bregman divergence analog of the standard
result for steepest descent in an arbitrary norm $\norm{\cdot}$,
giving convergence in the dual norm $\norm{\nablaLoss}_*$.
If the smoothness assumption is satisfied with constant $L$, 
we recover existing results in Euclidean norm,
\aligns{
	\fullmin\,
	\norm{\nablaLoss(\npt)}^2
	\leq 
	\frac{L}{T} \paren{\Loss(\npStart) - \Loss^*},%
}
as the $L$-smoothness of $\A$ implies 
the $\nicefrac{1}{L}$-strong convexity of $\AS$ 
and 
$\DS(\mptt, \mpt) \geq \nicefrac{1}{L} \, \norm*{\nablaLoss(\npt)}{}^2$.

The convergence in \aKL divergence, however, 
does not depend on additional smoothness assumptions and is a stronger guarantee
as it implies the probabilistic models being optimized stop changing.
This can not be directly guaranteed by small gradient norms, 
as differences in distributions do not only depend on the difference between parameters.
For example, how much a Gaussian distribution changes 
when changing the mean depends on its variance;
if the variance is small, the change will be big, 
but if the variance is large, the change will be comparatively smaller. 
This is illustrated in \cref{fig:grad-vs-kl},
and is not captured by gradient norms.

\begin{figure}[t]
\centering
	\centering
	\figscale{%
		\adjustbox{height=2.25cm}{%
			\adjustbox{scale={1}{.925}}{%
				\null
					\adjustbox{trim={.0\width} {.075\height} {.0\width} {.06\height},clip}{%
					\includegraphics[width=1.0\columnwidth]{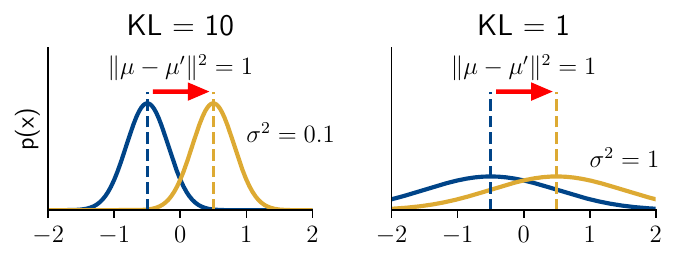}%
					}%
			}%
		}
	}
	\caption{
		The similarity between two Gaussians
		depends on their variance, even if it is fixed.
		The Euclidean distance between parameters, 
		and by extension gradient norms, 
		is a poor measure of stationarity as it ignores unchanged parameters. 
	}
	\label{fig:grad-vs-kl}
\end{figure}

\subsection{Connection to the Natural Gradient}
\label{sec:natural-decrement}
A useful simplification to interpret the divergence
is to consider the norm it is locally equivalent to.
By a second-order Taylor expansion, we have that 
\aligns{
	\DS(\mp + \delta, \mp) \approx \norm{\delta}_{\nabla^2 \AS(\mp)}^2,
}
\vspace{-1.70em}

where 
$\norm*{\delta}\raisebox{.5em}{$\scriptstyle2$}\m!\!\!%
\raisebox{-.25em}{$\scriptstyle\nabla^2 \AS(\mp)$} 
= \lin*{\delta, \nabla^2 \AS(\mp) \delta}$.
For exponential families, 
$\nabla^2 \AS$ 
is the inverse of the Fisher information matrix 
of the complete-data distribution $p(\data, \lat \cond \np)$,
\aligns{
	\nabla^2 \AS(\mp) = 
	I_{\data, \lat}(\np)^{-1} \m! = \ts\Expect[\data, \lat \sim p(\data, \lat \cond \np)]{\nabla^2 \log p(\data, \lat \cond \np)}.
}
The left side of \cref{cor:convergence-em-as-md}
is then, locally,
\alignn{\label{eq:natural-decrement}
	\DS(\mpt - \nablaLoss(\npt), \mpt) \approx 
	\norm{\nablaLoss(\npt)}_{I_{\data, \lat}(\npt)^{-1}}^2.
}
This quantity is the analog of the Newton decrement,
\aligns{
	\norm{\nablaLoss(\npt)}^2_{\nabla^2\Loss(\npt)^{-1}},	
}
used in the affine-invariant analysis of Newton's method
\citep{nesterov1994interior}. 
But \cref{eq:natural-decrement} is for the natural gradient 
in information geometry,
$I_{\data, \lat}(\npt)^{-1} \nablaLoss(\npt)$
\citep{amari2000methods}.
While the Newton decrement is invariant to affine reparametrizations, 
this ``natural decrement''
is also invariant to 
any homeomorphism.

\subsection{Invariant Local Linear Rates}
\label{sec:convex}

It was already established by \citet{dempster1977maximum}
that, asymptotically, the \aEM algorithm converges $r$-linearly, 
meaning that if it is in a convex region,\footnotemark
\footnotetext{The result of \citet{dempster1977maximum} concerned the convergence 
of the distance to the optimum, $\norm{\npt - \np^*}$,
but we use function values for simplicity, 
as it also applies.
}
\alignn{\label{eq:linear-rate-r}
	\Loss(\nptt)-\Loss^* \leq r(\Loss(\npt) - \Loss^*)
	&&
	\text{ for } r < 1,
}
near a strict minima $\np^*$, 
where the rate $r$ is determined by the amount of ``missing information''.
In this section, we strengthen \cref{cor:convex} to 
show this result extends to local but non-asymptotic rates.

The improvement ratio $r$ 
is determined by the eigenvalues 
of the missing information matrix $M$ 
at $\np^*$, defined as \citep{orchard1972missing}
\alignn{\label{eq:missing-information-matrix}
	M(\np^*) = I_{\data, \lat}(\np^*)^{-1} I_{\lat \cond \data}(\np^*),
}
where $I_{\data, \lat}(\np)$ and $I_{\lat\cond\data}(\np)$ are the Fisher information matrices 
of the complete-data distribution $p(\data, \lat \cond \np)$
and conditional missing-data distribution $p(\lat \cond \data, \np)$.
Intuitively this matrix measures how much information is missing
and how much easier the problem would be if we had access to the true 
values of the latent variables.
If $I_{\lat \cond \data}(\np)$ is small, there is little information 
to be gained from observing the latent variables 
as most of this information is already contained in $\data$.
But $I_{\lat \cond \data}(\np)$ is high if the known values of $\data$ 
do not constrain the possible values of $\lat$ and the problem 
is more difficult. 

The matrix $M(\np)$ is not a fixed quantity and evolves with the parameters $\np$.
In regions where we have a good model of the data, 
for example if we found well-separated clusters fit with Gaussian mixtures, 
there is little uncertainty about the latent variables (the cluster membership) and $M(\np)$ will be small.
However, $M(\np)$ is often large at the start of the optimization procedure.
The linear rate $r$ in \cref{eq:linear-rate-r}
is then determined 
by the maximum eigenvalue of the missing information,
$r = \lambda_{\max}\paren{M(\np^*)}$.
Linear convergence occurs if the missing information
at $\np^*$ is small, $M(\np^*) \prec 1$, otherwise $r$ can be larger than $1$.

This result, however, is only asymptotic and existing 
non-asymptotic linear rates rely on strong-convexity assumptions instead
\citep[e.g.][]{balakrishnan2017statistical}.
A twice-differentiable function $f$ is
$\alpha$-strongly convex if
\aligns{
	\nabla^2 f(\np) \succeq \alpha I
	&&\text{for } \alpha > 0,
}
which means the eigenvalues of $\nabla^2 f(\np)$
are bounded below by $\alpha$.
If $f$ is also $L$-smooth,
as defined in \cref{sec:em-and-mirror-descent},
a gradient-\aEM type of analysis gives that
\aligns{
	f(\nptt) - f^* \leq \paren{1-\frac{\alpha}{L}}\paren{f(\npt)-f^*}.
}
This implies \aEM 
converges linearly 
if it enters a 
smooth and strongly-convex
region.
However, in these works,
the connection to the ratio of missing information 
is lost and the rate is not invariant to reparametrization.

We showed in \cref{sec:em-and-mirror-descent}
that, instead of measuring smoothness in Euclidean norms, 
the \aEM objective is $1$-smooth relative to its 
log-partition function $\A$.
Likewise, we can characterize strong convexity relative to a reference function $h$ 
\citep{lu2018relatively}, requiring that 
\alignn{\label{eq:def-relative-sc}
	\nabla^2 \Loss(\np) \succeq \alpha \n!\nabla^2 h(\np)
	&& \text{ for }\alpha > 0.
}
For \aEM, where we care about 
strong convexity relative to the log-partition $\A$,
\emph{the relative strong-convexity parameter $\alpha$
is directly related to the missing information;}

\begin{restatable}{proposition}{restateStrongConvexity}
\label{prop:sc}
For exponential families, the \aEM objective is $\alpha$-strongly convex
relative to $\A$
on a %
region $\someset$ %
iff the missing information $M$ \peqref{eq:missing-information-matrix} 
satisfies
\aligns{
	\lambda_{\max}(M(\np)) \leq (1 - \alpha) && \text{ for all } \np \in \someset.
}
\end{restatable}
We provide a detailed proof in \cref{app:proof-em-expfam}
and give here the main intuition.
For exponential families, the Hessian of the surrogate $\Q{\np}{\np}$ 
coincides with the Hessian of $\A(\np)$ \peqref{eq:def-q-expfam},
which is the Fisher information matrix 
of the complete-data distribution, $I_{\data, \lat}(\np)$.
Using the decomposition of \cref{eq:em-as-loss-and-entropy},
the Hessian of the objective can be shown to be equal to 
\aligns{
	\nabla^2 \Loss(\np) = I_{\data, \lat}(\np) - I_{\lat \cond \data}(\np).
}
The definition of $\alpha$-strong convexity relative 
to $\A$ \peqref{eq:def-relative-sc} 
for \aEM 
is then 
equivalent to 
\aligns{
	I_{\data, \lat}(\np) - I_{\lat \cond \data}(\np) \succeq \alpha I_{\data, \lat}(\np).
}
Multiplying by $I_{\data, \lat}(\np)^{-1}$ and rearranging terms yields
\aligns{
	M(\np) = I_{\data, \lat}(\np)^{-1} I_{\lat \cond \data}(\np)  \preceq (1- \alpha) I.
	\tag*{\qedsymbol}
}
Convergence results for mirror descent on relatively $1$-smooth and 
$\alpha$-strongly convex functions \citep{lu2018relatively}
then directly give the following local linear rate.

\begin{restatable}{corollary}{convergenceEmAsMdStronglyConvex}
\label{cor:convergence-sc}
Under \ref{ass:1}--\ref{ass:3}, 
if \aEM is initialized in a locally convex region $\Theta$
with minimum $\Loss^*$
and the ratio of missing information is bounded,
$\lambda_{\max}(M(\np)) \leq r$,
\aligns{
	\Loss(\nptt) - \Loss^*
	\leq 
	r \paren{\Loss(\npt) - \Loss^*}.
}
\end{restatable}
\vspace{-.5em}
If the ratio of missing information goes to zero, as in the case of well-separated clusters for Gaussian mixtures with suitable initialization, then \aEM converges
superlinearly in the neighborhood of a solution
\citep{salakhutdinov2003optimization,xu1996convergence,ma2000asymptotic}.

\subsection{Generalized \aEM}
We now consider generalized \aEM schemes, 
which do not optimize the surrogate exactly in the \Mstep
but output an approximate (possibly randomized) update.
Given $\npt$, we assume we can solve the surrogate 
problem with some expected guarantee on the optimality gap,
$\Expect{\Qt{\nptt} - \QtStar}$,
where $\QtStar$ is the minimum value of the surrogate.
The \Mstep achieve $\Qt{\nptt} = \QtStar$, 
but it might be more efficient to solve the problem only partially,
or the \Mstep might be intractable.
We consider two types of guarantees for those two cases, 
multiplicative and additive errors.
\begin{enumerate}[label=\textbf{A\arabic*}]
	\setcounter{enumi}{3}
	\item \label{ass:multiplicative-error}
	{\bf Multiplicative error:}
	The approximate solution $\thetatt$ satisfies
	the guarantee that, for some $c \n!\in\n! (0,1]$,
\end{enumerate}
\vspace{-.8em}
\aligns{%
	\Expect{\Qt{\thetatt} - \QtStar}
	\leq 
	(1-c) \paren{\Qt{\thetat} - \QtStar}.
}
\vspace{-.2em}%
If $c = 1$, the algorithm is exact and $\thetatt$ minimizes the surrogate,
while for $c = 0$ there is no guarantee of progress.
An example of an algorithm satisfying this condition for mixture models 
would be the exact optimization of only one of the mixture components, 
chosen at random, like the \aECM algorithm of \citet{meng1993maximum}.
As the surrogate problem is separable among components, 
the guarantee is satisfied with $c = \nicefrac{1}{k}$, 
where $k$ is the number of clusters.
\vspace{-.1em}
\begin{restatable}{theorem}{restateMultiplicativeError}
\label{thm:multiplicative-error}
Under assumptions \ref{ass:1}--\ref{ass:3},
if the \aM-steps are solved up to $c$-multiplicative error (\ref{ass:multiplicative-error}),
\vspace{-.2em}
\aligns{
	\fullmin\,
	\Expect{\DS(\s(\npt), \mpt)}
	\leq 
	\frac{1}{c} \frac{\Loss(\npStart) - \Loss^*}{\maxiter}.
}
\end{restatable}
\vspace{-.5em}%
This can give speedup in overall time if some iterations 
can be made more than than $c$ times faster by leveraging the structure of the problem.

Multiplicative error, however, 
is a strong assumption if the closed-form solution is intractable. 
Instead, additive error is almost always satisfied. %
For example, although suboptimal for the reasons mentioned earlier, 
running \aGD with a line-search on the surrogate 
guarantees additive error if the objective is (locally) smooth.
\begin{enumerate}[label=\textbf{A\arabic*}]
	\setcounter{enumi}{4}
	\item \label{ass:additive-error}
	{\bf Additive error:}
	The algorithm returns a solution $\thetatt$ 
	with the guarantee that, in expectation, 
\end{enumerate}
\vspace{-.8em}%
\aligns{%
	\Expect{\Qt{\thetatt} - \QtStar}
	\leq 
	\epsilon_t.
}
\vskip-.75em%
If $\epsilon_t = 0$, the optimization is exact. 
Otherwise, the algorithm might not guarantee progress 
and the sequence $\epsilon_t$ needs to converge to 0 
for the iterations to converge.

\begin{restatable}{theorem}{restateAdditiveError}
\label{thm:additive-error}
Under assumptions \ref{ass:1}--\ref{ass:3},
if the \aM-step at step $t$ is 
solved up to $\epsilon_t$-additive error (\ref{ass:additive-error}),
\aligns{
	\fullmin\,
	\Expect{\DS(\s(\npt), \mpt)}
	\leq 
	\frac{\Loss(\npStart) - \Loss^*}{\maxiter}
	+ \frac{1}{T} \fullsum \epsilon_t.
}
\end{restatable}
\vspace{-.5em}
For example, for $\epsilon_t = \bigO(1/t)$, the rate reduces to $\bigO(\log(T)/T)$, 
but we recover a $\bigO(1/T)$ rate if the errors decrease faster, $\epsilon_t = \bigO(1/t^2)$.
As in \cref{sec:convex}, 
the results can be extended to give convergence in function values in a locally 
convex region.
The proofs of \cref{thm:multiplicative-error,thm:additive-error}
are deferred to \cref{app:proof-em-expfam}.

\subsection{\aEM for General Models}
\label{sec:em-general}

While the exponential family covers many applications of \aEM,
some are not smooth, in Euclidean distance or otherwise.
For example, in a mixture of Laplace distributions the gradient of the
surrogate is discontinuous (the Laplace distribution is not an exponential family).
In this case,
the progress of the \Mstep need not be related to the gradient 
and results similar to \cref{cor:convergence-em-as-md} do not hold.
To the best of our knowledge, 
there are no general non-asymptotic convergence results
for the general non-differentiable, non-convex setting
and all we can guarantee is asymptotic convergence,
as in the works of \citet{chretien2000kullback,tseng2004analysis}.

The tools presented here can still obtain partial answers for 
the Laplace mixture and similar examples. 
The analyses in previous sections considered the progress of the \Mstep, 
as is common in non-asymptotic literature.
We can instead view the \Estep as the primary driver of progress, 
as is more common in the asymptotic literature.
Assuming relative smoothness on the conditional distribution 
$p(\lat \cond \data, \np)$ only,
we derive in \cref{app:estep-analysis}
an analog of \cref{cor:convergence-em-as-md}
for stationarity on the latent variables,
rather than the complete-data distribution.
This guarantee is weaker, but the assumption holds more generally. 
For example, it is satisfied by any finite mixture, 
even if the mixture components are non-differentiable,
as for the Laplace mixture.

\makesection{Discussion}

Instead of assuming that the objective is smooth in Euclidean norm
and applying the methodology for the convergence of gradient descent, 
which does not hold even for the standard Gaussian mixture examples found in textbooks, 
we showed that \aEM for exponential families 
always satisfies a notion of smoothness relative to a Bregman divergence. 
In this setting, \aEM and its stochastic variants 
are equivalent to mirror descent updates.
This perspective leads to 
convergence rates that hold without additional assumptions 
and that are invariant to reparametrization.
We also showed how the ratio of missing information can be integrated in non-asymptotic convergence rates, 
and analyzed the use of approximate \Mstep{}s.
Although we focused on the \aMLE, \cref{app:priors} discusses \aMAP estimation.
We show that results similar to \cref{thm:convergence-em-as-md} on the convergence to stationary points in \aKL divergence still hold, 
with minor changes to incorporate the prior.
Viewing \aEM as a mirror descent procedure also highlights that it is a first-order method. 
It is thus susceptible to similar issues as classical first-order methods, 
such as slow progress in ``flat'' regions. However, flatness is measured in a different geometry (\aKL divergence) rather than the Euclidean geometry of gradient descent.

Beyond non-asymptotic convergence, 
smoothness relative to a \aKL divergence 
could be applied to extend statistical results,
such as that of \citet{daskalakis2016ten} to settings other than well-separated mixtures of Gaussians.
In addition to the \aEM algorithm, our results could be extended to variational methods, 
such as the works of \citet{hoffman2013svi,khan2016faster},
due to the similarity between the \aEM surrogate and the evidence lower-bound.

Stochastic variants of \aEM are becoming increasingly relevant 
as they allow the algorithm to scale to large datasets,
and recent recent work by 
\citet{chen2018stochastic,karimi2019global}
combined stochastic \aEM updates with variance reduction methods like SAG, SVRG, and MISO 
\citep{leroux2012sag,johnson2013svrg,mairal2015miso}.
But the analysis in those works take the view of \aEM 
as a preconditioned gradient step.
The resulting worst-case analysis
not only depends
on the smoothness constant, but the prescribed step-size is proportional to $\nicefrac{1}{L}$. For Gaussian mixtures with arbitrary initialization, this implies using a step-size of 0. 
Our results highlights the gap between \aEM and gradient-\aEM methods, using a combination of classic and modern tools from a variety of fields, and we hope that the tools developed here may help to fix this and similar practical issues.

\clearpage

\subsubsection*{Acknowledgements}
We thank the anonymous reviewers, whose comments helped improve the clarity of the manuscript.
We thank Frank Nielsen for pointing out that \ref{ass:1} needed to require \emph{steep} exponential families.
We thank
Si Yi (Cathy) Meng, 
Aaron Mishkin, 
and
Victor Sanches Portella
for providing comments on the manuscript and earlier versions of this work,
and for suggesting related material.
We are also grateful to 
Jason Hartford, 
Jonathan Wilder Lavington,
and 
Yihan (Joey) Zhou
for conversations 
that informed the ideas presented here. 

This research was partially supported by the Canada \acro{CIFAR} \acro{AI} Chair Program, 
the Natural Sciences and Engineering Research Council of Canada
(\acro{NSERC}) Discovery Grants \acro{RGPIN-2015-06068} 
and the \acro{NSERC} Postgraduate Scholarships-Doctoral Fellowship \acro{545847-2020}.

\subsubsection*{References}
\AtNextBibliography{\normalsize}
\printbibliography[heading=none]

\clearpage
\thispagestyle{empty}
\onecolumn
\appendixtitle{\thetitle\\Supplementary Materials}
\appendix

\makeatletter
\def\paragraph{
\@startsection{paragraph}
  {4}
  {\z@}
  {1.5ex plus 0.5ex minus .2ex}
  {-.5em}
  {\normalsize\bfseries}
}
\makeatother

\section*{Organization of the supplementary material}
\vspace{.5em}
\begin{description}[labelwidth=2.5cm,labelindent=10pt,leftmargin=3.0cm,align=left]
	\item[\cref{app:recap}] \nameref{app:recap}
	\item[\cref{app:proof-equivalence-em-md}] \nameref{app:proof-equivalence-em-md}
	\item[\cref{app:priors}] \nameref{app:priors}
	\item[\cref{app:proof-em-expfam}] \nameref{app:proof-em-expfam}
	\item[\cref{app:estep-analysis}] \nameref{app:estep-analysis}
\end{description}

\vfill

\begin{table}[ht]
\centering
\caption{Summary of notation and acronyms}
\vspace{.5em}
\renewcommand{\arraystretch}{1.05}%
\begin{tabular}{p{.15\textwidth}p{.15\textwidth}p{.55\textwidth}}
\toprule
Context & Symbol & \\
\midrule
Data & $\data$, $\lat$ & 
Observed ($\data$) and missing $(\lat)$, or latent, variables.
\\[.5em]
Parameters & $\np, \phi \in \Omega$ & 
(Natural) Parameters of the model and set of valid parameters.
\\
& $\mp$ & Equivalent mean parameters.
\\[.5em]
\aEM & $\Loss(\np)$ & 
Objective function, the negative log-likelihood $\mneg\log p(\data\cond\np)$.
\\
 & $\Q{\np}{\phi}$ & 
Surrogate objective optimized by the \Mstep.
\\[.5em]
\multirow{2}{.15\textwidth}{Exponential families }
& $\stats(\data, \lat)$ & Sufficient statistics.
\\
 & $\A(\np)$, $\AS(\np)$ & Log-partition function and its convex conjugate.
\\
 & $\D(\phi, \np)$ & Bregman divergence induced by the function $A$.
\\[.5em]
\multirow{2}{.15\textwidth}{Fisher information }
& $I_{\data,\lat}(\np)$ & 
Fisher information matrix of 
the distribution $p(\data,\lat\cond\np)$.
\\
& $I_{\lat\cond\data}(\np)$ & 
Fisher information matrix of 
the distribution $p(\lat\cond\data,\np)$.
\\[.5em]
Optimization & $t = 1, \ldots, T$ & Iteration counter and total iterations.
\\
\bottomrule
\end{tabular}

\vspace{1em}

{Acronyms:}

\vspace{1em}

\begin{minipage}[t]{.8\textwidth}
\begin{multicols}{2}
\begin{description}[labelwidth=1.0cm,labelindent=10pt,leftmargin=1.0cm,align=left,format={\normalfont}]
	\item[\aMLE] maximum likelihood estimate
	\item[\aMAP] maximum a posteriori estimate
	\item[\aNLL] negative log-likelihood
	\item[\aEM] expectation-maximization
	\item[\aGD] gradient descent
	\item[\aFIM] Fisher information matrix
	\item[\aKL] Kullback-Leibler
\end{description}
\end{multicols}
\end{minipage}
\end{table}

\vfill

\clearpage
\section{Supplementary material for \cref{sec:em-and-expfam}:\newline 
\sectionNameEmAndExpFam{}}
\label{app:recap}

\begingroup
\setlength{\parskip}{.5em}
This section extends on the background given in \cref{sec:em-and-expfam}
and give additional details and properties on
\begin{description}[labelwidth=2.5cm,labelindent=10pt,leftmargin=3.0cm,align=left]
	\item[\cref{app:recap-em}] \nameref{app:recap-em}
	\item[\cref{app:recap-ef}] \nameref{app:recap-ef}
	\item[\cref{app:recap-bregman}] \nameref{app:recap-bregman}
	\item[\cref{app:recap-fisher}] \nameref{app:recap-fisher}
	\item[\cref{app:recap-md}] \nameref{app:recap-md}
\end{description}

\subsection{Expectation-Maximization}
\label{app:recap-em}

This section gives additional details on the derivation of the \aEM surrogate 
and some of the perspective taken on the algorithm in the literature.
\citet{lange2000optimization,mairal2013optimization} 
view \aEM as a majorization-minimization algorithm to develop a general analysis 
and extend it to other problems.
\citet{chretien2000kullback,tseng2004analysis} view it instead as a 
proximal point method in Kullback-Leibler divergence
to study its asymptotic convergence properties.
Finally, \citet{csiszar1984information,neal1998view}
take an alternating minimization procedure view of the algorithm.
\citeauthor{csiszar1984information} use it to analyze its convergence properties 
while \citeauthor{neal1998view} develop an incremental variant.
This last perspective is the one presented by \citet{wainwright2008graphical}, 
viewed as a variational method.

The form of the algorithm presented in the main text is the one used 
by \citet{dempster1977maximum}.
The negative log-likelihood (\aNLL) $\Loss(\phi)$, surrogate $\Q{\np}{\phi}$ and entropy term $\Hent{\np}{\phi}$
are defined as
\aligns{
	\Loss(\np) = -\log p(\data\cond\np),
	&&
	\Q{\np}{\phi} = -\medint\log p(\data,\lat\cond\phi) \, p(\lat\cond\data,\np)\dif{\lat},
	&&
	\Hent{\np}{\phi} = -\medint\log p(\lat\cond\data,\phi) \, p(\lat\cond\data,\np)\dif{\lat}.
}
They obey the decomposition $\Q{\np}{\phi} = \Loss(\phi) + \Hent{\np}{\phi}$.
To show this, we use the fact that $\int p(\lat\cond\data,\np) \dif{\lat} = 1$, and
\aligns{
	\Loss(\phi) 
	= - \log p(\data\cond\phi) 
	= 
	- \log p(\data\cond\phi) 
	\cdot
	\medint p(\lat\cond\data,\np)\dif{\lat}
	=
	- \medint \log p(\data\cond\phi) \, p(\lat\cond\data,\np)\dif{\lat}.
}
Along with the chain rule, $p(\data,\lat\cond\phi) = p(\lat\cond\data,\phi) \, p(\data\cond\phi)$, 
we get
\aligns{
	\Loss(\phi)
	&=
	-\medint \log p(\data\cond\phi) \, p(\lat\cond\data,\np)\dif{\lat}
	\\[-1.5em]
	&=
	-\medint \log\paren{\frac{p(\data,\lat\cond\phi)}{p(\lat\cond\data,\phi)}} \, 
	p(\lat\cond\data,\np)\dif{\lat}
	=
	\overbrace{
		-\medint \log p(\data,\lat\cond\phi) \, p(\lat\cond\data,\np) \dif{\lat}
	}^{\Q{\np}{\phi}}
	+ \overbrace{
		\medint \log p(\lat\cond\data,\phi) \, p(\lat\cond\data,\np) \dif{\lat}
	}^{-\Hent{\np}{\phi}}
}

\subsubsection*{From a Majorization-Minimization perspective}
A majorization-minimization procedure in the sense of \citet{lange2000optimization} 
is an iterative procedure to optimize the objective $\Loss$.
Given the current estimate of the parameters $\npt$, 
we first find a majorant, an upper bound $f_t$ that it is tight at $\npt$,
$\Loss(\phi) \leq f_t(\phi)$ and 
$\Loss(\npt) = f_t(\npt)$.
We then minimize $f_t$ to obtain the new estimate $\nptt$.
As $f_t$ is an upper bound on the objective, 
$\nptt$ is guaranteed to be an improvement if it is an improvement on $f_t$.

The typical derivation of \aEM in this setting involves
expressing the \aNLL
as the marginal of the complete-data likelihood, 
multipliying the integrand by $\frac{p(\lat\cond\data,\np)}{p(\lat\cond\data,\np)}$
and using Jensen's inequality, $\mneg\log(\Expect{x}) \leq \mneg\Expect{\log(x)}$,
\aligns{
	\Loss(\phi)
	&= - \log \medint p(\data, \lat \cond \phi) \dif{\lat}
	\\
	&= - \log \medint p(\lat\cond\data,\np) \frac{p(\data, \lat \cond \phi)}{p(\lat \cond\data,\np)} \dif{\lat}
	\\[-1.5em]
	&\leq
	- \medint \log\paren{\frac{p(\data, \lat \cond \phi)}{p(\lat \cond\data,\np)}}
	\, p(\lat\cond\data,\np)  \dif{\lat}
	= 
	\overbrace{
	- \medint \log p(\data, \lat \cond \phi)
	\, p(\lat\cond\data,\np)  \dif{\lat}}^{\Q{\np}{\phi}}
	+ \overbrace{
	\medint \log p(\lat \cond\data,\np) \, p(\lat \cond\data,\np) \dif{\lat}
	}^{- \Hent{\np}{\np}}.
}
It gives that the surrogate $\Q{\np}{\cdot}$ is an upper bound on the objective, up to a constant,
$\Loss(\phi) \leq \Q{\np}{\phi} + \const$.
The surrogate $\Q{\np}{\cdot}$ itself is not a majorant, 
as $\Q{\np}{\np} = \Loss(\np) + \Hent{\np}{\np}$.
The difference, however, is not relevant for optimization as it does not depend on $\phi$.
If we define instead the surrogate as $\Qprime{\np}{\phi} = \Q{\np}{\phi} - \Hent{\np}{\np}$, 
we get
\aligns{
	\Qprime{\np}{\phi} = \Loss(\phi) + \Hent{\np}{\phi} - \Hent{\np}{\np}.
	&&\text{ and }&&
	\Loss(\np) = \Qprime{\np}{\np}
}
The two formulations of the surrogate share the same minimizers as they differ by an additive constant.
\endgroup

\subsubsection*{From a proximal point perspective}
The definition of $\Qprime{\np}{\cdot}$ also gives the proximal point perspective
used by \citet{chretien2000kullback,tseng2004analysis} 
to discuss the asymptotic convergence properties of \aEM.
The differences of entropy terms is a \aKL divergence;
\aligns{
	\Hent{\np}{\phi} - \Hent{\np}{\np}
	&= -\int \log p(\lat\cond\data,\phi) \, p(\lat\cond\data,\np)\dif{\lat}
	+ \medint \log p(\lat\cond\data,\np) \, p(\lat\cond\data,\np)\dif{\lat},
	\\[-.25em]
	&= - \medint \log \paren{\frac{p(\lat\cond\data,\phi)}{p(\lat\cond\data,\np)}} \, p(\lat\cond\data,\np) \dif{\lat}
	\quad = \KL{p(\lat\cond\data,\np)}{p(\lat\cond\data,\phi)}.
}
The \aEM iterations can then be expressed as minimizing $\Loss$ and a \aKL proximity term,
\aligns{\ts
	\nptt 
	= \arg\min_\theta Q'_{\npt}(\np)
	=
	\arg\min_\np \braces{
		\Loss(\np) + \KL{p(\lat \cond \data, \npt)}{p(\lat\cond\data, \np)}
	}.
}
\citet{amid2020divergence} used this view
to extend stochastic versions of \aEM beyond exponential families.

\subsubsection*{From an alternating minimization perspective}
The expression in terms of a \aKL divergence 
also gives the alternating minimization approach used by 
\citet{csiszar1984information} to show asymptotic convergence,
and by \citet{neal1998view} to justify partial updates.
This is the variational approach presented by \citet{wainwright2008graphical}.
For a distribution $q$ on the latent variables, 
parametrized by $\phi$, the objective function is equivalent to
\aligns{\textstyle
	\Loss(\np)
	= - \log p(\data \cond \np)
	= - \log p(\data \cond \np) + \min_\phi \KL{q(\lat \cond \phi)}{p(\lat \cond \data, \np)}
}
if $q$ is sufficiently expressive and we can minimize the \aKL divergence exactly.
The parameters $\phi$ and $\np$ need not be defined on the same space, 
as $\phi$ only controls the conditional distribution over the latent variables 
and $\np$ controls the complete-data distribution.
We can write the \aEM algorithm as alternating optimization 
on the augmented objective function
\aligns{
	\Aug(\np,\phi) = - \log p(\data \cond \np) + \KL{p(\lat\cond\data, \phi)}{p(\lat \cond \data, \np)}
	&&\text{ such that }&& \ts \Loss(\np) = \min_\phi \Aug(\np,\phi).
}
The \aE and \aM steps then correspond to 
\aligns{
	\text{\aE-step:}&&\ts
	\phi_{t+1} = \arg\min_\phi \Aug(\np_t, \phi),
	&&
	\text{\aM-step:}&&\ts
	\np_{t+1} = \arg\min_\np \Aug(\np, \phi_{t+1}).
}
We will return to this perspective in \cref{app:estep-analysis}
to analyse the progress of the \Estep.

\subsubsection*{Gradients and Hessians}
From the equivalence between $\Q{\np}{\phi}$ and $\Qprime{\np}{\phi}$ up to constants, 
they share the same gradient as the \aNLL at $\np$, as
\aligns{
	\nabla \Q{\np}{\phi}\cond_{\phi=\np}
	= \nabla \Qprime{\np}{\phi}\cond_{\phi=\np}
	= \nablaLoss(\np)
	+ 
	\underbrace{\nabla_\phi \KL{p(\lat\cond\data,\np)}{p(\lat\cond\data,\phi)} \cond_{\phi=\np}}_{=0},
}
\vspace{-2em}

if they are differentiable. Similarly, their Hessian is 
\aligns{
	\nabla^2 \Q{\np}{\phi}\cond_{\phi=\np}
	= \nabla^2 \Qprime{\np}{\phi}\cond_{\phi=\np}
	= \nabla^2 \Loss(\np)
	+ 
	\nabla_\phi^2 \KL{p(\lat\cond\data,\np)}{p(\lat\cond\data,\phi)} \cond_{\phi=\np}.
}

\subsubsection*{Invariance to homeomorphisms}

The invariance of the \aEM update to homeomorphisms 
is a direct result of the exactness of the \Mstep. 
A homeomorphism between two parametrizations $(\np, \mp)$
is a continuous bijection $f$ with continous inverse $f^{-1}$, such that 
$\np = f(\mp)$ and $\mp = f^{-1}(\np)$.
Although we use the same notation as the mean and natural parameters, 
$\np$ and $\mp$ can be any parametrization.
Letting $(\npt,\mpt)$ be the current iterates,
the \aEM update in parameters $\np$ or $\mp$ yields
\aligns{\ts
	\nptt \in \arg\min_\np \Q{\npt}{\np}
	&&\ts
	\mptt \in \arg\min_\mp \Q{f(\mpt)}{f(\mp)}.
}
If $\Q{\np}{\cdot}$ is strictly convex, it has a unique minimum and 
$\nptt = f(\mptt)$, $\mptt = f^{-1}(\nptt)$.
Otherwise, $(f, f^{-1})$ defines a bijection between the possible updates.
While the update in some parametrizations might be 
easier to implement, the update to the probabilistic model is the same 
regardless of the parametrization.

\clearpage
\subsection{Exponential families}
\label{app:recap-ef}

For a detailed introduction on exponential families, 
we recommend the work of \citet[]{wainwright2008graphical}.

An distribution $p(x\cond\np)$ is in the exponential family 
with natural parameters $\np$
if it has the form 
\aligns{
	p(x \cond \np)
	= h(x) \exp\paren{\lin{\stats(x), \np} - \A(\np)}
	&&\text{}&&
	\mneg\log p(x \cond \np)
	= \A(\np) - \lin{\stats(x), \np} - \log h(x),
}
where $h$ is the base measure, $\stats$ are the sufficient statistics, 
and $\A$ is the log-parition function.
We did not discuss the base measure $h$ in the main text;
it is necessary to define the distribution
but does not influence the optimization as it does not depend on $\np$.
This can be seen from the gradient and Hessian of the \aNLL;%
\aligns{
	\nabla \mneg\log p(x \cond \np)
	= \dA(\np) - \stats(x)
	&&\text{ and }&&
	\nabla^2 \mneg\log p(x\cond\np) = \nabla^2 \A(\np).
}
\subsubsection*{Examples: Bernoulli and univariate Gaussian}
For a binary $\data \in \{0,1\}$, 
the Bernoulli distribution $p(x\cond\pi) = \pi^x(1-\pi)^x$
is an exponential family distribution with
\aligns{
	&h(x) = 1
	&&
	\stats(x) = x
	&&
	\np = \log\paren{\frac{\pi}{1-\pi}}
	&&
	\A(\np) = \log(1 + e^\np) = -\log(1-\pi).
\intertext{For $\data \in \R$, 
the Gaussian $p(x\cond\mu,\sigma^2) = \frac{1}{\sqrt{2\pi\sigma^2}}\exp(-(x-\mu)^2/2\sigma^2)$
is an exponential family distribution with}
	&h(x) = \frac{1}{\sqrt{2\pi}}
	&&
	\stats(x) = \brackets{x, x^2}
	&&
	\np = \brackets{\frac{\mu}{\sigma^2}, -\frac{1}{2\sigma^2}}
	&&
	\A(\np) = -\frac{\theta_1^2}{4\theta_2}
	- \frac{1}{2}\log\nleft\vert-\frac{1}{2\theta_2}\nright\vert
	= \frac{\mu^2}{2\sigma^2} + \log\sigma.
}
\subsubsection*{The log-partition function and mean parameters}
Given the base measure $h$ and sufficient statistics function $\stats$, 
the log-partition function $\A$ is defined 
such that the probability distribution is valid and integrates to 1,
\aligns{
	\begin{aligned}
	1 = \medint p(x\cond\np)\dif{x}
	&= \medint h(x) \exp\paren{\lin{\stats(x), \np} - \A(\np)} \dif{x} ,
	\\[-.5em]
	&= \exp\paren{-\A(\np)} \medint h(x) \exp\paren{\lin{\stats(x), \np}} \dif{x}
	\end{aligned}&&\implies&&
	\A(\np) &= \log \medint h(x) \exp\paren{\lin{\stats(x), \np}} \dif{x}
}
This formulation gives that the log-partition function is convex 
and its gradient yields the expected sufficient statistics produced by the model, 
$\dA(\np) = \Expect[p(x\cond\np)]{\stats(x)}$
\aligns{
	\dA(\np)
	&= 
	\nabla \log \medint h(x) \exp\paren{\lin{\stats(x), \np}} \dif{x}
	= 
	\frac{1}{\int h(x) \exp\paren{\lin{\stats(x), \np}} \dif{x}} 
	\nabla \medint h(x) \exp\paren{\lin{\stats(x), \np}} \dif{x},
	\\
	&=
	\exp(-\A(\np))
	\nabla \medint h(x) \exp\paren{\lin{\stats(x), \np}} \dif{x}
	= \exp(-\A(\np)) \medint h(x) 
	\exp \stats(x) \paren{\lin{\stats(x), \np}} \dif{x}
	= 
	\medint \stats(x) \, p(x\cond\np) \dif{x}.
}
If the log-partition function $\A$ is strictly convex, 
the exponential family is said to be minimal 
and there is a bijection between $\np$ and the expected sufficient statistics.
The expected sufficient statistics give an equivalent way to parametrize the model, 
called the mean parameters, which are denoted $\mp$. 
The gradient $\dA$ maps the natural to 
the mean parameters, $\mp = \dA(\np)$.
The inverse mapping is the gradient of the convex conjugate of $\A$,
\aligns{\ts
	\AS(\mp) = \sup_\np \braces{\lin{\np, \mp} - \A(\np)}.
}
We then get the bijection $\mp = \dA(\np)$ and $\np = \dAS(\mp)$.
The Hessians of $\A$ and $\AS$ are also inverses of each other.
This can be seen from the fact that the composition 
$\np = \dAS(\dA(\np))$ is the identity, and 
\aligns{
	\nabla \brackets{\dAS(\dA(\np))} =
	\nabla^2 \AS(\mu) \nabla^2 \A(\np) = I.
}
The minimality of the exponential family, or the strict convexity of $\A$,
ensures both $\nabla^2\A$ and $\nabla^2\AS$ are invertible.

\subsubsection*{For Expectation-Maximization}
When the complete-data distribution $p(\data,\lat\cond\np)$ is in the exponential family, 
the \Mstep has a simple expression as the surrogate $\Q{\np}{\cdot}$
depends on the data only through the expected sufficient statistics at $\np$, 
\aligns{
	\Q{\np}{\phi}
	=
	- \medint\log p(\data, \lat \cond \phi) \, p(\lat \cond \data, \np) \dif{\lat}
	=
	- \lin{\textstyle\Expect[p(\lat \cond \data, \np)]{\stats(\data, \lat)}, \phi} + \A(\phi).
}
Writing the expected sufficient statistics as $\statso(\np) = \Expect[p(\lat \cond \data, \np)]{\stats(\data, \lat)}$,
the gradients of the surrogate and \aNLL are
\aligns{
	\nabla \Q{\np}{\phi} 
	=
	- \statso(\np) + \nabla \A(\phi)
	&&\text{ and }&&
	\nablaLoss(\np) 
	= \nabla \Q{\np}{\np} = - \statso(\np) + \nabla \A(\np).
}

\clearpage
\subsection{Bregman divergences}
\label{app:recap-bregman}
For an overview of Bregman divergences 
in clustering algorithms
and their relation with exponential families, 
we recommend the work of \citet{banerjee2005clustering}.

\begin{minipage}[t]{.48\textwidth}
Bregman divergence are a generalization of squared Euclidean 
distance based on convex functions.
For a function $h$, $\Breg{h}(\np, \phi)$ 
is the difference between the function at $\np$ 
and its linearization constructed at $\phi$,
\aligns{
	\Breg{h}(\np, \phi) 
	= h(\np) - h(\phi) - \lin{\nabla h(\phi), \np - \phi}.
}
This is illustrated in \cref{fig:bregman}.
The simplest example of a Bregman divergence is the Euclidean distance, 
which is generated by setting $h(\np) = \frac{1}{2}\norm{\np}^2$, 
such that %
\end{minipage}\hfill%
\begin{minipage}[t]{.48\textwidth}
\null~
\vspace{-3em}
\begin{center}
\includegraphics[width=.9\textwidth]{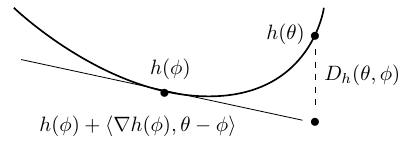}
\end{center}
\vspace{-1.4em}
\captionof{figure}{Illustration of the Bregman divergence of a convex function $h$ as the difference between the linearization of the function and its value.}
\label{fig:bregman}
\end{minipage}
\aligns{
	\hspace{-5em}
	\Breg{h}(\np, \phi) 
	&= h(\np) \,\,\,\,\, - h(\phi) \,\,\,\,\, - \lin{\nabla h(\phi), \np - \phi},
	\\
	&= \frac{1}{2}\norm{\np}^2 - \frac{1}{2}\norm{\phi}^2 
	- \lin{\phi, \np - \phi}
	\quad\quad\quad = \frac{1}{2}\norm{\np}^2 - \lin{\phi, \np} + \frac{1}{2}\norm{\phi}^2
	\quad = \frac{1}{2}\norm{\np - \phi}^2.
}
Other examples of Bregman divergences include
\aligns{
	&\text{Weighted Euclidean/Mahalanobis distance:}
	&&x \in \R^d
	&&\ts
	h(x) = \frac{1}{2} \lin{x, Ax}
	&&\ts
	D_h(x, y) = \frac{1}{2}\norm{x-y}_A^2
	\\
	&\text{Kullback-Leibler divergence on the simplex:}
	&&\pi \in \Delta^{\m!d}
	&&\ts
	h(\pi) = \sum_{i=1}^d \pi_i \log \pi_i
	&&\ts
	D_h(\tau,\pi) = \sum_{i=1}^d \pi_i \log\paren{\frac{\pi_i}{\tau_i}}
}
\subsubsection*{General properties}
The Euclidean example is not representative of general Bregman divergences, 
as they lack some properties of metrics.
They are not necessarily symmetric (in general, $\Breg{h}(\np, \phi) \neq \Breg{h}(\phi, \np)$)
and do not satisfy the triangle inequality.
The Bregman divergence is convex in its first argument, as it reduces to $h(\np)$ 
and a linear term, but needs not be convex in its second argument.
The gradients and Hessian with respect to the first argument are 
\aligns{
	\begin{aligned}
	\nabla_\np \Breg{h}(\np, \phi) 
	= \nabla_\np \brackets{h(\np) - h(\phi) - \lin{\nabla h(\phi), \np - \phi}}
	= \nabla h(\np) - \nabla h(\phi)
	\end{aligned}
	&&\text{ and }&&
	\nabla_\np^2 \Breg{h}(\np, \phi) &= \nabla^2 h(\np).
}
Bregman divergences statisfy a generalization of the Euclidean decomposition, 
called the three point-property;
\aligns{
	\text{Euclidean:}
	&&
	\norm{a - c}^2 = \norm{a - b + b - c}^2
	&= \norm{a - b}^2 + 2 \lin{a - b, b - c} + \norm{b - c}^2,
	\\
	\text{Bregman divergence:}
	&&
	\Dh(a, c) &= \Dh(a,b) + \lin{a-b, \dh(b) - \dh(c)} + \Dh(b,c).
}
This property can be directly verified 
by expanding $\Dh(a,b) = h(a) - h(b) - \lin{\nabla h(b), a -b}$,
\aligns{
	\begin{array}{llcccllllllllllllllllllll}
	&\Dh(a,b) 
	&\!\!\!\!\!\!\!\!+\!\!\!\!\!\!\!\!& \lin{a-b, \dh(b) - \dh(c)} 
	&\!\!\!\!\!\!\!\!+\!\!\!\!\!\!\!\!& \Dh(b,c)
	\\
	=\quad 
	&h(a) - h(b) - \lin{\dh(b), a-b}
	&\!\!\!\!\!\!\!\!+\!\!\!\!\!\!\!\!& \lin{\dh(b), a- b} - \lin{\dh(c), a- b}
	&\!\!\!\!\!\!\!\!+\!\!\!\!\!\!\!\!& h(b) - h(c) - \lin{\dh(c), b-c}
	\\
	=\quad 
	&h(a) 
	&&
	&&\hphantom{h(b)} - h(c) - \lin{\dh(c), a-c}
	= \Dh(a, c)
	\end{array}
}
The Bregman divergence induced by $h$ and its convex conjugate $h^*$ satisfy the following relation,
\aligns{
	D_h(x, y) = D_{h^*}(\nabla h(y), \nabla h(x)).
}
The convex conjugate of a function $h$ is $h^*(\mu) = \sup_{\theta} \braces{\lin{\theta, \mu} - h(\theta)}$,
and if $h$ is strictly convex and differentiable, the supremum is attained at $\mu = \nabla h(\theta)$,
creating a mapping from the domain of $h$ to the range of its gradient. 
The inverse mapping can be found by taking the bi-conjugate (the conjugate of the conjugate), which recovers $h = (h^*)^*$;
$h(\theta) = \sup_{\mu}\braces{\lin{\mu, \theta} - h^*(\mu)}$, and the supremum is attained at $\theta = \nabla h^*(\mu)$.

\subsubsection*{For exponential families}
For an exponential family $p(x\cond\np)$, 
the Bregman divergence induced by the log-partition function $\A$ 
is the Kullback-Leibler divergence between the distributions given by the parameters
\aligns{
	\KL{p(x\cond\phi)}{p(x\cond\np)}
	= \medint \log\paren{\frac{p(x\cond\phi)}{p(x\cond\np)}} \, p(x\cond\phi) \dif{\lat}
	&= 
	\medint \paren{\lin{\stats(x), \phi} - \A(\phi) - \lin{\stats(x), \np} + \A(\np)} p(x\cond\phi) \dif{\lat},
	\\[-.5em]
	&= \ts
	\A(\np) - \A(\phi)
	+
	\lin{\Expect[p(x\cond\phi)]{\stats(x)}, \phi - \np},
	\\[.5em]
	&=
	\A(\np) - \A(\phi)
	+
	\lin{\dA(\phi), \phi - \np} 
	\quad = \D(\np, \phi).
}

\clearpage
\subsection{Fisher information matrices}
\label{app:recap-fisher}

For an introduction to Fisher information 
in the context of the \aEM algorithm 
and its connection to the ratio of missing information, 
we recommend the work of \citet[\S3.8--3.9]{mclachlan2007em}.

For a probability distribution parametrized by $\np$,
$p(x \cond \np)$, the Fisher information 
is a measure of the information 
that observing some data $x$ would provide about the parameter $\np$.
The Fisher information matrix (\aFIM) is
\aligns{\ts
	I(\np) 
	\;=\;
	\nabla^2_{\phi} \KL{p(x \cond \np)}{p(x \cond \phi)} \cond_{\phi = \np}
	\;=\;
	\Expect[p(x \cond \np)]{\nabla^2 \mneg\log p(x \cond \np)}
	\;=\;
	\Expect[p(x \cond \np)]{\nabla \log p(x \cond \np) \nabla \log p(x \cond \np)^\top},
}
where all expressions are equivalent.
As we will have to distinguish between the information of different distributions, 
we define the following 
notation for the distributions 
$p(x \cond \np)$,
$p(\data,\lat\cond \np)$, and 
$p(\lat \cond \data, \np)$;
\aligns{
	I_{\data}(\np) \m!= 
	\!\!\!\Expect[p(\data \cond \np)]{\nabla^2 \mneg\log p(\data \cond \np)},
	&&
	I_{\data,\lat}(\np) \m!= 
	\!\!\!\Expect[p(\data, \lat \cond \np)]{\nabla^2 \mneg\log p(\data, \lat \cond \np)},
	&&
	I_{\lat\cond\data}(\np) \m!= 
	\!\!\!\Expect[p(\lat \cond \data, \np)]{\nabla^2 \mneg\log p(\lat \cond \data, \np)}.
}
The first two do not depend on data
as $\data$ and $\lat$ are sampled from the probabilistic model. 
The conditional \aFIM $I_{\lat\cond\data}(\np)$ depends on the observed data $x$
as the expectation is with respect to $p(\lat\cond\data,\np)$.

The Fisher information depends on the parametrization of the distribution.
Let us write $I_{x\cond\np}$ and $I_{x\cond\mp}$ for the Fisher information of two equivalent parametrizations, $(\np, \mp)$,
and $(f, f^{-1})$ be the homeomorphism such that $\np = f(\mp)$ and $\mp = f^{-1}(\np)$.
The information matrices obey
\aligns{
	I_{x\cond\mp}(\mp) = \jac f(\mp)^\top \, I_{x\cond\np}(\np) \, \jac f(\mp),
}
where $\jac f$ is the Jacobian of $f$. 
Although we use $\np$ and $\mp$, those parametrizations need not be the natural and mean parametrization for this property to hold.
This is shown most easily by using the outer-product form;
\aligns{
	I_{x\cond\mp}(\mp) 
	&= \ts
	\Expect[p(x\cond f(\mp))]{\nabla_\mp \log p(x\cond f(\mp)) \nabla_\mp \log p(x\cond f(\mp))^\top},
	\\
	&= \ts
	\Expect[p(x\cond f(\mp))]{\jac f(\mp)^\top \nabla_\np \log p(x\cond \np) \nabla_\np \log p(x\cond\np)^\top \jac f(\mp)},
	\\
	&= \ts
	\jac f(\mp)^\top \Expect[p(x\cond \np)]{\nabla \log p(x\cond \np) \nabla \log p(x\cond\np)^\top} \jac f(\mp)
	= \jac f(\mp)^\top \, I_{x\cond\np}(\np) \, \jac f(\mp).
}

{\bf For an exponential family distribution} $p(x\cond\np)$, the \aFIM is also equal to the Hessian of the \aNLL, as
\aligns{\ts
	I(\np) 
	\;=\;
	\Expect[p(x \cond \np)]{\nabla^2 \mneg\log p(x \cond \np)}
	\;=\;
	\Expect[p(x \cond \np)]{\nabla^2 \A(\np)}
	\;=\; \nabla^2 \A(\np).
}
For the natural and mean parameters $(\np, \mp)$, 
applying the reparametrization property to $(\dA, \dAS)$
along with the fact that $I(\np) = \nabla^2 \A(\np)$ and $\dA(\np) = [\dAS(\mp)]^{-1}$ 
gives that $I_{x\cond\mp}(\mp) = I_{x\cond\np}(\np)^{-1}$, as
\aligns{
	I_{x \cond \mp}(\mp) = \nabla^2 \AS(\mp) I_{x\cond\np}(\np) \nabla^2 \AS(\mp)
	= \nabla^2 \AS(\mp) \nabla^2 \A(\np) \nabla^2 \AS(\mp) = \nabla^2 \AS(\mp).
}

{\bf For Expectation-Maximization,} if the complete-data distribution $p(\data,\lat\cond\np)$ 
is in the exponential family, 
the Hessian of the surrogate and objective are 
\vspace{-.5em}
\aligns{
	\nabla^2 \Q{\np}{\np} = \nabla^2 \A(\np) = I_{\data,\lat}(\np),
	&&
	\begin{aligned}\nabla^2 \Loss(\np) 
	&= \nabla^2 \Q{\np}{\np} - \nabla_\phi^2 \KL{p(\lat\cond\data,\np)}{p(\lat\cond\data,\phi)} \cond_{\phi=\np}
	\\&=
	I_{\data,\lat}(\np) - I_{\lat\cond\data}(\np).
	\end{aligned}
}
This follows from the definition of $\Q{\np}{\cdot}$ (\cref{app:recap-em})
and the properties of exponential families (\cref{app:recap-ef}). 

\subsubsection*{Natural gradients}
The gradient is a measure of the direction of steepest increase, 
where steepest is defined with respect to the Euclidean distance between the parameters.
When the parameters of a function also define a probability distribution, 
the natural gradient \citep{amari2000methods}
is the direction of steepest increase, where steepest is instead measured
by the \aKL divergence between the induced distributions.
The natural gradient is obtained by preconditioning the gradient with the inverse of the \aFIM 
of the relevant distribution, $I(\np)^{-1}\nablaLoss(\np)$.

For exponential families, the gradient with respect to the natural parameters $\np$ 
is the natural gradient with respect to the mean parameters $\mp$.
Letting $\Loss_d(\mp) = \Loss(\dAS(\mp))$ be the objective express in mean parameters, 
we have
\aligns{
	\nablaLoss_d(\mp) = \nabla^2\AS(\mp) \nablaLoss(\np) 
	= [\nabla^2\A(\np)]^{-1} \nablaLoss(\np)
	&&
	\nabla \Loss(\np) 
	= \nabla^2 \A(\np) \nablaLoss_d(\mp)
	= [\nabla^2\AS(\mp)]^{-1} \nablaLoss_d(\mp)
}
This implies the mirror descent update $\mptt = \mpt - \nablaLoss(\npt)$
is a natural gradient descent step in mean parameters 
when the mirror map $\A$ 
is the log-partition function of an exponential family 
\citep{raskutti2015information}.
The view of \aEM as a natural gradient update
was already used by \citet{sato1999fastlearning}
to justify a stochastic variant.

\clearpage
\subsection{Mirror descent, convexity, smoothness and their relative equivalent}
\label{app:recap-md}

For a more thorough coverage of mirror descent, 
we recommend the works of \citet{nemirovski1983,beck2003mirror}.
For an introduction on convexity, smoothness and strong convexity, 
we recommend the work of \citet{nesterov2013introductory}.
For their relative equivalent, see \citet{bauschke2017descent,lu2018relatively}.

The traditional gradient descent algorithm to optimize a function $f$
can be expressed as the minimization of the linearization of 
$f$ at the current iterates $\npt$ and a Euclidean distance proximity term
depending on the step-size $\gamma$,
\aligns{\ts
	\nptt = \arg\min_\np \braces{f(\npt) + \lin{\nabla f(\npt), \np - \npt} 
	+ \frac{1}{2 \gamma} \norm{\np - \nptt}^2}.
}
As the surrogate objective is convex, the update is found by taking the derivative 
and setting it to zero;
\aligns{\ts
	\nabla f(\npt) + \frac{1}{\gamma} (\nptt - \npt) = 0
	&&\implies&&
	\nptt = \npt - \gamma \nabla f(\npt).
}
The mirror descent algorithm is an extension where the Euclidean distance is 
replaced by a Bregman divergence,
\aligns{\ts
	\np' = \arg\min_\phi \braces{f(\np) + \lin{\nabla f(\np), \phi - \np} + \frac{1}{\gamma} \Breg{h}(\phi, \np)}.
}
Setting $h(\np) = \frac{1}{2}\norm{\np}^2$ recovers the gradient descent surrogate.
The stationarity condition gives the update
\aligns{\ts
	\nabla f(\npt) + \frac{1}{\gamma} \paren{\nabla h(\nptt) - \nabla h(\npt)} = 0
	&&\implies&&
	\nabla h(\nptt) = \nabla h(\npt) - \gamma \nabla f(\npt).
}
Or, equivalently, the update can be written in the dual parametrization $\mp = \nabla h(\np)$,
\aligns{
	\mptt = \mpt - \gamma \nabla f(\npt).
}
The mirror descent update applies the gradient step to the dual parameters 
instead of the primal parameters $\np$.
In the mirror descent literature, 
the reference function $h$ is called the mirror function or mirror map.

\subsubsection*{Smoothness and strong convexity}
The gradient descent update with an arbitrary constant step-size $\gamma$ 
is not guaranteed to make progress on the original function $f$,
at least not without additional assumptions. 
A common assumption is that the function $f$ is smooth, 
meaning that its gradient is Lipschitz with constant $L$,
\aligns{
	\norm{\nabla f(\np) - \nabla f(\phi)} \leq L\norm{\np - \phi},
	\quad \text{ for any } \np,\phi.
}
The $L$-smoothness of $f$ implies the following upper bound holds, 
\aligns{
	f(\phi) 
	\leq 
	f(\np) + \lin{\nabla f(\np), \phi - \np)} + \frac{L}{2}\norm{\np-\phi}^2
	\quad \text{ for any } \np,\phi.
}
Setting $\gamma \leq \frac{1}{L}$ ensures the surrogate optimized by gradient descent 
is an upper bound on $f$ and leads to progress.
If the objective function is also $\alpha$-strongly convex, 
meaning the following lower bound holds,
\aligns{
	f(\np) + \lin{\nabla f(\np), \phi - \np} - \frac{\alpha}{2}\norm{\np-\phi}^2
	\leq
	f(\phi) 
	\quad \text{ for } \alpha > 0 \text{ and any } \np,\phi,
}
gradient descent converges at a faster, linear rate.
This definition recovers convexity in the case $\alpha = 0$ 
and is otherwise stronger.
If $f$ is twice differentiable, $\alpha$-strong convexity and $L$-smoothness 
are equivalent to 
\aligns{
	\alpha I \preceq \nabla^2 f(\np) \preceq L I 
	\quad \text{ for all } \np.
}
Here, $\preceq$ is the Loewner ordering on matrices, 
where $A \preceq B$ if $B - A$ is positive semi-definite,
meaning the minimum eigenvalue of $B - A$ is larger than or equal to zero. 

\subsubsection*{Relative smoothness and strong convexity}
Relative $L$-smoothness and $\alpha$-strong convexity 
provide an analog of smoothness and strong-convexity for mirror descent.
They are defined relative to a reference function $h$, 
such that the following lower and upper bound hold
\aligns{
	f(\np) + \lin{\nabla f(\np), \phi - \np)} - \alpha \Breg{h}(\phi, \np)
	\leq 
	f(\phi) 
	\leq 
	f(\np) + \lin{\nabla f(\np), \phi - \np)} + L \Breg{h}(\phi, \np)
	\quad 
	\text{ for } 0 < \alpha \leq L \text{ and any } \np, \phi.
}
Alternatively, if $f$ and $h$ are twice differentiable,
those conditions are equivalent to 
\aligns{
	\alpha \nabla^2 h(\np) \preceq \nabla^2 f(\np) \preceq L \nabla^2 h(\np)
	\quad \text{ for all } \np.
}
In the case $h(\np) = \frac{1}{2}\norm{\np}^2$, 
we recover the standard definition 
of Euclidean smoothness and strong-convexity.

\newpage
\section{Supplementary material for \cref{sec:em-and-mirror-descent}:\newline 
\sectionNameEmAndMD{}}
\label{app:proof-equivalence-em-md}
\vspace{-.5em}
The connection between \aEM and mirror descent presented in the main paper, and in more details here,
is similar to the insights of previous works that describe \aEM using exponential families or Bregman divergences.

This sufficient statistics update of \aEM
predates the work of \citet{dempster1977maximum} 
and was used in a variety of scenarios such as
online \citep{sato1999fastlearning,cappe2009online},
incremental \citep{neal1998view} and 
variance reduced \citep{chen2018stochastic,karimi2019global} variants of \aEM.
The view of \aEM as minimizing divergences was developped by \citet{csiszar1984information,amari1995em}.
This apprach was used by \citet{chretien2000kullback,tseng2004analysis}
to derive asymptotic convergence results for \aEM, while \citet{amid2020divergence} 
provide a divergence-based description of online \aEM, beyond exponential families.
Closest to our work is the generalization of clustering algorithms of \citet{banerjee2005clustering}, 
based on minimizing Bregman divergences.

Although a result of the same ideas, the equivalence between mirror descent and \aEM does not seem to have been formally stated.
The key distinction from previous work is our focus on non-asymptotic convergence rates,
where current proofs use the machinery of gradient descent (Euclidean smoothness) to describe \aEM.
Our contribution is to show that, viewed as mirror descent,
the relative smoothness framework of \citet{bauschke2017descent,lu2018relatively} 
yields non-asymptotic convergence rates for \aEM without additional, unrealistic assumptions.

The connection between the progress of \aEM in \aKL divergence
and the ``natural decrement'' (Eq.~\ref{eq:natural-decrement})
builds on the connection between mirror and natural gradient descent 
in natural and mean parameterization, noted by \citet{sato1999fastlearning} and \citet{raskutti2015information} 
in the context of \aEM and mirror descent, respectively.

This section gives additional details on 
the relationship between \aEM and mirror descent 
and the 1-relative smoothness of \aEM.
We restate in longer form the proof of \cref{prop:equivalence};
\begin{thmbox}
\restateEmAsMd*
\end{thmbox}
\vspace{-.5em}
\begin{proof}[Proof of \cref{prop:equivalence}]
Recall the decomposition of the surrogate
in terms of the objective and entropy term,
$\Q{\np}{\phi} = \Loss(\phi) + \Hent{\np}{\phi}$
in \cref{eq:em-as-loss-and-entropy}. 
It gives
\aligns{
	\Loss(\phi) - \Loss(\np)
	= \Q{\np}{\phi} - \Q{\np}{\np} + 
	\Hent{\np}{\np} - \Hent{\np}{\phi},
}
where $\Hent{\np}{\np} - \Hent{\np}{\phi} \leq 0$
as $\Hent{\np}{\phi}$ is minimized at $\phi=\np$.
We will show that for exponential families, 
\aligns{
	\Q{\np}{\phi} - \Q{\np}{\np}
	=
	\lin{\nablaLoss(\np), \phi - \np} + \D(\phi, \np),
}
which implies the upper-bound in 
\cref{eq:mirror-descent-upper-bound}
and that its minimum matches that of $\Q{\np}{\phi}$.

If the complete-data distribution is in the exponential family,
the surrogate in natural parameters is 
\aligns{
	\Q{\np}{\phi}
	&=
	- \medint\log p(\data, \lat \cond \phi) \, p(\lat \cond \data, \np) \dif{\lat},
	\\[-.2em]
	&=
	- \medint\brackets{\lin{\stats(\data, \lat), \phi} - \A(\phi) } \, p(\lat \cond \data, \np) \dif{\lat}
	=
	- \lin{\textstyle\Expect[p(\lat \cond \data, \np)]{\stats(\data, \lat)}, \phi} + \A(\phi).
}
For simplicity of notation, we write $\s(\np)$ for the expected sufficient statistics $\Expect[p(\lat \cond \data, \np)]{\stats(\data, \lat)}$
(while the $\s(\theta)$ depends on $\data$ and we could write $\s(\np, \data)$,
we ignore it as the same $\data$ is always given to $\s$).
We will use the definition of the Bregman divergence 
and the fact that the gradient of the surrogate matches the gradient of the objective,
\aligns{
	\D(\phi,\np) = \A(\phi) - \A(\np) - \lin{\dA(\np), \phi - \np},
	&&\text{ and }&&
	\nablaLoss(\np) = \nabla\Q{\np}{\np} = \dA(\np) - \s(\np).
}
Expanding $\Q{\np}{\phi} - \Q{\np}{\np}$, we have
\begin{fleqn}\aligns{
	\hspace{2em}
	\Q{\np}{\phi} - \Q{\np}{\np}
	&=%
	- \lin{\s(\np), \phi - \np}
	+ \A(\phi) - \A(\np),
	\\
	&=
	- \lin{s(\np), \phi - \np}
	+ \lin{\dA(\np), \phi - \np}
	+ \A(\phi) - \A(\np)
	- \lin{\dA(\np), \phi - \np},
	\tag{$\pm \lin{\dA(\np), \phi - \np}$}
	\\
	&=
	- \lin{s(\np) - \dA(\np), \phi - \np}
	+ \D(\phi, \np),
	\tag{$\D(\phi, \np) = \A(\phi) - \A(\np) - \lin{\dA(\np), \phi - \np}$}
	\\
	&=
	\lin{\nablaLoss(\np), \phi - \np}
	+ \D(\phi, \np).
	\tag*{\qedhere}
}\end{fleqn}
\end{proof}
For completeness, we present an alternative derivation that 
relies on additional material in Appendices \ref{app:recap-em}--\ref{app:recap-md}.
That the \Mstep is a mirror descent step 
can be seen from the stationary point of 
$\Q{\np}{\phi}$ and \cref{eq:mirror-descent-upper-bound}, 
\aligns{
	\A(\phi) = \s(\np)
	&&\text{ and }&&
	\A(\phi) = \A(\np) - \nablaLoss(\np) = \s(\np).
}
To show the upper bound holds, 
we can use the expansion of the objective as 
$\Loss(\phi) = \Q{\np}{\phi} - \Hent{\np}{\phi}$ to get 
\aligns{
	\nabla^2 \Loss(\np) = \nabla^2 \Q{\np}{\phi}
	\nabla^2 \KL{p(\lat\cond\np)}{p(\lat\cond\phi)}
	= \nabla^2\A(\phi) - I_{\lat\cond\data}(\phi),
}
where $I_{\lat\cond\data}(\phi)$ is the \aFIM of $p(\lat\cond\data,\phi)$.
As Fisher information matrices are positive semi-definite, 
we get that $\nabla^2 \Loss(\np) \preceq \nabla^2 \A(\np)$, 
establishing the 1-smoothness of \aEM relative to $\A$
and the upper bound in \cref{eq:mirror-descent-upper-bound}. 

\subsubsection*{Equivalence between stochastic \aEM and stochastic mirror descent}
\label{sec:em-expfam-stoch}
We now look at variants of \aEM based on stochastic approximation
and show they can be cast as stochastic mirror descent.
We focus on the online \aEM of \citet{cappe2009online}, 
but it also applies to the incremental and stochastic versions of 
\citet{neal1998view,sato1999fastlearning,delyon1999convergence}.
As in the deterministic case, this result is limited to exponential families
and does not extend to the divergence-based description of online \aEM of \citet{amid2020divergence}.

The stochastic version of the \aEM update 
uses only a subset of samples per iteration
to compute the \Estep
and applies the \Mstep to the average of the sufficient statistics 
observed so far.
We assume we have $n$ independent samples for the observed variables, 
$x_1, \ldots, x_n$,
such that the objective factorizes as 
\aligns{
	\Loss(\np) = \sum_{i=1}^n \Loss_i(\np) = \sum_{i=1}^n \log p(\data_i \cond \np)
}
Defining the individual expected sufficient statistics as 
$s_{i}(\npt) = {\textstyle \Expect[p(\lat \cond \data_{i}, \npt)]{\stats(x_{i}, z)}}$,
the online \aEM algorithm updates a running average of sufficient statistics using 
a step-size $\gamma_t$
\alignn{\label{eq:online-em-update}
	\mp_{t+1} = (1- \gamma_t) \mp_t 
	+ \gamma_t s_{i_t}(\npt),
	&&
	\text{ where } i_t \sim U[n].
}
With step-sizes $\gamma_t = \nicefrac{1}{t}$, 
the mean parameters at step $t$ are the average of the observed sufficient statistics,
\aligns{
	\mp_\maxiter = (1- \gamma_\maxiter) \mp_{\maxiter-1}
	+ \gamma_\maxiter \s_{i_\maxiter}(\np_\maxiter)
	= \frac{\maxiter-1}{\maxiter} \mp_{\maxiter-1} + \frac{1}{\maxiter} \s_{i_\maxiter}(\np_\maxiter)
	= \frac{1}{\maxiter} \fullsum \s_{i_t}(\npt).
}
The natural parameters are then updated with $\npt = \dAS(\mpt)$.

\begin{restatable}{proposition}{restateOemAsSmd}
\label{prop:equivalence-stochastic}
The online \aEM algorithm (Eq.~\ref{eq:online-em-update})
is equivalent to the stochastic mirror descent update
\alignn{\label{eq:smd}
	\nptt &= 
	\arg\min_\phi
	\braces{\lin{\nablaLoss_{i_t}(\npt), \phi - \npt}
	+ \frac{1}{\gamma_t} \D(\phi, \npt)},
	&&\text{ with } i_t \sim U[n].
}
\end{restatable}
\begin{proof}[Proof of \cref{prop:equivalence-stochastic}]
We show the equivalence of one step, assuming they select the same index $i_t$.
The online \aEM update  (Eq.~\ref{eq:online-em-update})
guarantees the natural and mean parameters match, $\npt = \dAS(\mu_t)$, 
and the update to $\nptt$ is 
\aligns{
	\nptt 
	= \dAS\big( 
		(1-\gamma_t) \mu_t 
		+ \gamma_t s_{i_t}(\npt)
	\big).
}
where $s_{i}(\np) = \Expect[p(\lat \cond \data_{i}, \np)]{\stats(\data_{i}, \lat)}$.
The stationary point of \cref{eq:smd}, on the other hand, ensures 
\aligns{\begin{aligned}
	0 = 
	\nablaLoss_{i_t}(\npt) + \frac{1}{\gamma_t} 
	\paren{\dA(\nptt) - \dA(\npt)}
	&&\implies&&
	\nptt = \dAS(\dA(\npt) - \gamma_t \nablaLoss_{i_t}(\npt)).
\end{aligned}}
As in the proof of \cref{prop:equivalence} (\cref{app:proof-equivalence-em-md}),
using that the gradient of the loss and the surrogate match, 
\aligns{
	\nablaLoss_{i_t}(\npt)
	= - s_{i_t}(\npt)
	+ \dA(\npt),
}
we get that both update match,
\aligns{
	\dA(\npt) - \gamma_t \nablaLoss_{i_t}(\npt)
	(1-\gamma_t) \dA(\npt) + \gamma_t 
	s_{i_t}(\npt)
	&&\implies&&
	\mptt = 
	(1-\gamma_t) \mpt + \gamma_t 
	s_{i_t}(\npt).
	\tag*{\qedhere}
}
\end{proof}

\newpage
\section{Supplementary material for
\cref{sec:assumptions}:\newline \sectionNameAssumptions{}}
\label{app:priors}

This section gives additional details on the assumptions discussed 
in \cref{sec:assumptions} and shows the derivation 
for maximum a posteriori (\aMAP) estimation with \aEM under a conjugate prior.
We first mention how \ref{ass:3} implies that the \aEM iterates are well-defined
and introduce notation to discuss proper conjugate priors for exponential families.
We then show that a proper prior implies that the surrogate optimized by \aEM leads to well-defined solutions 
and satisfies \ref{ass:3}, and end with showing how an equivalent of \cref{thm:convergence-em-as-md} holds for \aMAP.

\paragraph{\ref{ass:3} guarantees that the update are well defined.}
Consider fitting the variance $\sigma^2 > 0$ of a Gaussian.
The update is ill-defined if it goes to the boundary ($\sigma^2 \inleq 0$), or diverges ($\sigma^2 \m!\to\m! \infty$).
Here is how \ref{ass:3} avoids those cases.

The exponential family assumptions imply that the surrogate $\Q{\npt}{\cdot}$ optimized during the \Mstep is convex. 
Its domain, $\Omega$, is an open set.
\ref{ass:3} constrains the sub-level sets $\Omega_{\npt} = \{ \phi \in \Omega : \Q{\npt}{\phi} \leq \Q{\npt}{\np} \}$ 
to be compact (closed and bounded). 
As the minimum of the surrogate is contained in any sub-level set, 
it must be finite (as the sub-level sets are bounded) 
and contained strictly in $\Omega$ (as the sub-level sets are closed).

\paragraph{Proper conjugate priors.}
We first discuss exponential families, without the added complexity of \aEM.
In the main text, we used $\data$ to denote the entire dataset. 
To discuss priors, it is useful to consider the dataset as $n$ i.i.d. observations 
$x_1, \ldots, x_n$ from a (minimal, regular) exponential family, 
with negative log-likelihood (\aNLL)
\aligns{
	p(\data_i \cond \theta) \propto \exp(\lin{T(x_i), \theta} - A(\theta)),
	&&
	\text{\acro{NLL}}(\theta) = - \sum_{i=1}^n \log p(x_i \cond \theta) = 
	\sum_{i=1}^n A(\theta) - \lin{T(x_i), \theta}.
}
For exponential families, 
parametrizing the prior 
by a strength $n_0 > 0$
and the sufficient statistics $m_0$ we expect to observe a priori, 
the conjugate prior that leads to the same form for the posterior is
\aligns{
	p(\theta \cond m_0, n_0) \propto \exp \paren{ \lin{m_0, \theta} - n_0 A(\theta)}.
}
The regularized objective of adding the \aNLL and the prior is then, up to a multiplicative constant of $n+n_0$,
\alignn{\label{eq:app-map}
	\Loss(\theta) = 
	\frac{1}{n+n_0} 
	\paren{
		- \sum_{i=1}^n \log p(x_i \cond \theta)
		- \log p(\theta \cond m_0, n_0)
	}
	=
	A(\theta) - \lin{\bar m, \theta},
	&& \text{with }
	\bar m = \frac{m_0 + \sum_{i=1}^n T(x_i) }{n+n_0}.
}
To discuss proper priors, we need to discuss the constraint set $\Omega$ in more details. 
For a $d$-dimensional, regular, minimal exponential family, the set of valid natural parameters is defined from the log-partition function as
$\Omega = \{\theta \in \R^d \cond A(\theta) < \infty\}$.
The equivalent set of mean parameters, through the bijection $(\dA, \dAS)$, is 
\aligns{
	\ts
	\setM = \{\mu \in \R^d \cond \exists \theta \in \Omega : \Expect[p(x\cond\theta)]{T(x)} = \mu \}
	\quad (\text{or } \setM = \dA(\Omega)),
}
the image of $\Omega$ through $\dA$ \citep[][Theorem 3.3]{wainwright2008graphical}.
For the prior to be proper, 
the expected sufficient statistics under the prior $m_0$ need to be in the interior of 
$\setM$ \citep[][Theorem 1]{diaconis1979conjugate}.

\paragraph{\aMAP solutions are well defined.}
The sufficient statistics $T(x_i)$ could lie on the boundary of $\setM$, which is why the \aMLE is sometimes ill-defined. 
For example, estimating the covariance of a Gaussian from one sample leads to $\sigma^2 = 0$.
However, if the prior is proper, $m_0 \in \setM$ then the average $\bar m = \frac{1}{n+n_0} (\sum_{i=1}^n T(x_i) + m_0)$ 
will also be in $\setM$.
By convexity, the \aMAP is at the stationary point of \cref{eq:app-map}, 
$\dA(\bar \theta) = \bar m$ and $\bar \theta = \dAS(\bar m)$ 
will be in $\Omega$. 

For completeness, let us show that this also implies \ref{ass:3}. 
As $\Omega$ is convex, it is sufficient to show that $\Loss(\theta) \to \infty$ 
from any direction $v \in \R^d$ starting from $\bar\theta$,
leading to the sequence $\theta(t) = \bar \theta + t v$ for $t > 0$.

\begin{itemize}
	\item 
If $\theta(t)$ crosses the boundary of $\Omega$, $\Loss(\theta(t)) \to \infty$ due to the log-partition function.

Let $t_b$ be the finite crossing point. The parameters $\theta(t_b)$ and the inner product $\lin{\bar m, \theta(t_b)}$ are also finite.
But since the boundary of $\Omega$ is defined by $A(\theta) < \infty$, 
by (lower-semi-)continuity of $A$, $\lim_{t\to t_b^-} A(\theta(t)) = \infty$.
	\item
If $\theta(t)$ does not cross a boundary,
$\Loss(\theta(t)) \to \infty$ by strict convexity.

Consider the restriction of $\Loss$ to the line spanned by $v$, 
$f(t) = \Loss(\theta(t))$ for $t > 0$. 
By the properties of $\Loss$, $f(t)$ is strictly convex and minimized at $t=0$.	
Let $t_0 > 0$ be an arbitrary point. By strict convexity,
\aligns{
	f(t) > f(t_0) + f'(t_0) (t - t_0)
	&&
	\text{ and }
	&&
	f'(t_0) > 0
}
for some finite $f(t_0)$ and $f'(t_0)$. 
Taking the limit of the lower bound as $t \to \infty$ gives that $\lim_{t \to \infty} \Loss(\theta(t)) = \infty$. 	
\end{itemize}

For background on the constraint sets of parameters exponential families,
we recommend \citet[][\S3.4]{wainwright2008graphical}.
For a geometric view on priors in Bregman divergences, 
see \citet{agarwal2010geometric}.

\newcommand{\mapLoss}{\Loss_{\scriptscriptstyle\mathrm{MAP}}}

\paragraph{\aEM with a prior.}
We now consider the analysis of \aEM with a proper conjugate prior 
if the full-data distribution $p(\data, \lat \cond \np)$ is in the exponential family.
Assuming that the observed and latent variables can be partitioned into i.i.d. pairs $(\data_i, \lat_i)$, 
as is the case for example with Gaussian mixture models, 
the likelihood for a full observation is
\aligns{
	p(\data_i, \lat_i \cond \np) \propto \exp(\lin{T(x_i, z_i), \np} - A(\np)).
}
A conjugate prior on $\np$ will have the same form as above, 
\aligns{
	p(\np \cond m_0, n_0) \propto \exp(\lin{m_0, \np} - n_0 A(\np)),
}
and the \aMAP--\aEM objective will have the form (up to the normalization constant $n+n_0$)
\aligns{
	\mapLoss(\np) &= - \frac{1}{n+n_0}\paren{ \sum_{i=1}^n \log p(\data \cond \np) - \log p(\np \cond m_0, n_0) }.
}
Applying the same upper bounds as in the \aMLE case, we can define the \aMAP--\aEM surrogate as 
\aligns{
	\tilde Q_{\np}(\phi) 
	&= \frac{1}{n+n_0} 
	\paren{
		\sum_{i=1}^n \int \log p(x_i \cond z_i, \phi) p(z_i \cond x_i, \np) \dif{z}
		+ n_0 A(\np) - \lin{m_0, \np}
	},
	\\
	&=
	\frac{1}{n+n_0}\paren{\ts
		\sum_{i=1}^n A(\phi) - \lin{\Expect[p(z_i \cond x_i, \np)]{T(x_i, z_i)}, \np}
		+ n_0 A(\phi) - \lin{m_0, \phi}
	}.
\intertext{Writing $\bar{s}(\np) = \sum_{i=1}^n \Expect[p(z_i \cond x_i, \np)]{T(x_i, z_i)}$
for the sum of sufficient statistics, the surrogate is }
	&=
	A(\phi) - \lin{\bar{m}(\np), \phi},
	\quad \text{ where } 
	\bar m(\np) = \frac{\bar{s}(\np) + m_0}{n+n_0}.
}
Ignoring the rescaling by $n+n_0$, this only changes the original surrogate by adding a linear term.
The rescaled objective is still $1$-smooth\footnote{%
Without rescaling, the \aMLE and \aMAP objectives would be $n$-smooth 
and $(n+n_0)$-smooth relative to $A$.
While rescaling changes the constants, the resulting algorithm is the same;
running \aGD with step-size $\gamma$ on a function $f$ is equivalent to a step-size $\gamma / C$ on $f' = C f$.
}
relative to $A$, and the results derived for \aMLE still hold for \aMAP,
up to minor variations.
Writing $\Loss$ and $\mapLoss$ for the non-regularized \aMLE and regularized \aMAP objectives,
the equivalent of \cref{thm:convergence-em-as-md} includes the prior in the optimality gap;
\begin{restatable}{proposition}{convergenceEmAsMdMap}
\label{thm:convergence-em-as-md-map}
Under assumptions \ref{ass:1}--\ref{ass:3}, 
\aEM for exponential family distributions 
with a proper conjugate prior $p(\theta \cond m_0, n_0) \propto \exp(\lin{m_0, \theta} - n_0A(\theta))$
converges at the rate 
\aligns{
	\fullmin\,
	\KL{p(\data, \lat \cond \nptt)}{p(\data, \lat \cond \npt)}
	&\leq 
	\frac{\mapLoss(\npStart) - \mapLoss(\theta^*)}{T}
	\tag*{(where $\theta^*$ is a minimum of $\mapLoss$)}
	\\
	&=
	\frac{\Loss(\npStart) - \Loss(\theta^*)}{T}
	+ \frac{
	\log p(\np^* \cond m_0, n_0) 
	- \log p(\npStart \cond m_0, n_0) 
	}{T}.
}
\end{restatable}
The proof follows the same steps as \cref{thm:convergence-em-as-md},
and similar variants hold 
for the locally convex (\cref{cor:convex}) and strongly-convex (\cref{cor:convergence-sc}) cases.
To relate the convergence of the successive iterates of \cref{thm:convergence-em-as-md-map} to stationarity, 
a similar development as for \cref{cor:convergence-em-as-md} with the notation introduced above gives
\begin{restatable}{corollary}{convergenceEmAsMdTwoMap}
\label{cor:convergence-em-as-md-map}
Under assumptions \ref{ass:1}--\ref{ass:3}, 
with a proper conjugate prior $p(\theta \cond m_0, n_0) \propto \exp(\lin{m_0, \theta} - n_0A(\theta))$,
\aligns{
	\fullmin\,
	\DS\paren{\frac{\bar{s}(\np_t) + m_0}{n+n_0}, \mpt}
	\leq 
	\frac{\mapLoss(\npStart) - \mapLoss(\theta^*)}{\maxiter}.
}
The average of the prior and observed sufficient statistics $\frac{\bar{s}(\np_t) + m_0}{n+n_0}$ and 
the mean parameters $\mpt$ are the two parts of the regularized gradient, 
$\nablaLoss_{\scriptscriptstyle\mathrm{MAP}}(\npt) = \dA(\npt) - \frac{\bar{s}(\np_t) + m_0}{n+n_0}$,
and $\DS\paren{\frac{\bar{s}(\np_t) + m_0}{n+n_0}, \mpt} = 0$ implies $\nablaLoss_{\scriptscriptstyle\mathrm{MAP}}(\npt) = 0$.
\end{restatable}

\newpage
\section{Supplementary material for
\cref{sec:convergence-results}:\newline \sectionNameConvergenceEM{}}
\label{app:proof-em-expfam}

This section presents additional details and proofs for the results in \cref{sec:convergence-results};
\begin{description}[labelwidth=2.5cm,labelindent=10pt,leftmargin=3.0cm,align=left]
	\item[\cref{app:subsec-convergence-em}] \nameref{app:subsec-convergence-em}
	\item[\cref{app:subsec-natural-decrement}] \nameref{app:subsec-natural-decrement}
	\item[\cref{app:subsec-convergence-gem}] \nameref{app:subsec-convergence-gem}
	\item[\cref{app:subsec-relative-sc}] \nameref{app:subsec-relative-sc}
	\item[\cref{app:subsec-convergence-em-convex}] \nameref{app:subsec-convergence-em-convex}
\end{description}

\subsection{Convergence of \aEM to stationary points (\cref{thm:convergence-em-as-md,cor:convergence-em-as-md})}
\label{app:subsec-convergence-em}
\begin{thmbox}
\convergenceEmAsMd*
\end{thmbox}
\begin{proof}[Proof of \cref{thm:convergence-em-as-md}]
Assumptions \ref{ass:1}--\ref{ass:3} 
ensure that the updates are well defined. 
\ref{ass:1} ensures the mapping $(\dA, \dAS)$ is well defined 
and the update $\npt \to \nptt$ is unique. 
\ref{ass:2} ensures the objective is lower-bounded by some value $\Loss^*$
and \ref{ass:3} ensures that, if the parameters are restricted 
to an open set $\Omega$, the updates remain in $\Omega$ 
as long as $\npStart \in \Omega$.
\cref{prop:equivalence}
then gives that a step from $\npt$ to $\nptt$ satisfies
\aligns{
	\Loss(\nptt)  \leq \Loss(\npt) + \lin{\nablaLoss(\npt), \nptt - \npt}
	+ \D(\nptt, \npt).
}
As $\nptt$ is selected to minimize the upper bound,
it is at a stationary point. Using that $\nabla\D(\np,\npt) = \dA(\np) - \dA(\npt)$,
\aligns{\begin{aligned}
	\nabla_{\nptt} \{
	\lin{\nablaLoss(\npt), \nptt - \npt}
	+ \D(\nptt, \npt) \} = 0
	&& \implies && 
	\nablaLoss(\npt) + \dA(\nptt) - \dA(\npt) = 0.
\end{aligned}}
Substituting 
$\nablaLoss(\npt)$ for $\dA(\npt) - \dA(\nptt)$ in 
the upper bound 
and using the definition of Bregman divergences
\aligns{\D(\nptt,\npt) = \A(\nptt) - \A(\npt) - \lin{\dA(\npt), \nptt - \npt},}
gives the simplification
\aligns{
	\Loss(\nptt) 
	&\leq
	\Loss(\npt) 
	+ \lin{\nablaLoss(\npt), \nptt - \npt} 
	+ \D(\nptt, \npt)
	\\
	&=
	\Loss(\npt) 
	+ \lin{\dA(\npt) - \dA(\nptt), \nptt - \npt} 
	+ \D(\nptt, \npt),
	\\
	&=
	\Loss(\npt) 
	- \lin{\dA(\nptt), \nptt - \npt} 
	+ \A(\nptt) - \A(\npt)
	\quad =
	\Loss(\npt) - \D(\npt, \nptt).
}
Reorganizing the inequality, we have that 
\aligns{
	\D(\npt, \nptt) \leq \Loss(\npt) - \Loss(\nptt).
}
Summing over all iterations $t = 1, \ldots, T$ 
and dividing by $T$ gives the result, 
\aligns{
	\fullmin
	\D(\npt, \nptt)
	\leq
	\frac{1}{\maxiter}
	\fullsum
	 \D(\npt, \nptt)
	\leq \frac{1}{\maxiter} \fullsum \Loss(\npt) - \Loss(\nptt)
	= \frac{\Loss(\npStart) - \Loss(\npT)}{\maxiter}.
}
Using the lower-bound on the objective function, $\Loss(\npT) \geq \Loss^*$, finishes the proof.
\qedhere
\end{proof}
\begin{thmbox}
\convergenceEmAsMdTwo*
\end{thmbox}
\begin{proof}[Proof of \cref{cor:convergence-em-as-md}]
The proof follows from \cref{thm:convergence-em-as-md}
and the form of the update. We have that
\aligns{
	&\text{the update ensures}
	&
	\nablaLoss(\npt) &= \dA(\npt) - \dA(\nptt),
	\\
	&\text{the gradient is}
	&
	\nablaLoss(\npt) &= \dA(\npt) - \s(\npt),
	\\
	&\text{the Bregman divergence satisfies}
	&
	\D(\npt,\nptt) &= \DS(\dA(\nptt), \dA(\npt)).
} 
Using the mapping between natural and mean parameters, we get 
\aligns{
	\D(\npt,\nptt) 
	= \DS(\mptt, \mpt)
	= \DS(\mpt - \nablaLoss(\npt), \mpt)
	= \DS(\s(\npt), \mpt).
	\tag*{\qedhere}
}
\end{proof}

\subsection{Natural decrement (\cref{sec:natural-decrement})}
\label{app:subsec-natural-decrement}
For a small perturbation $\delta$, the Bregman divergence is well approximated 
by its second-order Taylor expansion
\aligns{
	\DS(\mp + \delta, \mp) 
	= 
	\underbrace{\DS(\mp,\mp)}_{=0}
	+ \, \lin*{\underbrace{\nabla_{\mp'} \DS(\mp', \mp)\cond_{\mp'=\mp}}_{=0}, \delta}
	+
	\frac{1}{2} \lin*{\delta, \underbrace{ \nabla_{\mp'}^2 \DS(\mp', \mp)\cond_{\mp'=\mp} }_{\nabla^2\AS(\mp)}\delta}
	+ o(\norm{\delta}^3)
	\approx \frac{1}{2}\norm{\delta}_{\nabla^2\AS(\mp)}^2.
}
Using that $\nabla^2\AS(\mp) = [\nabla^2\A(\np)]^{-1} = I_{\data,\lat}(\np)^{-1}$
(see \cref{app:recap-fisher})
and $\delta=\nablaLoss(\np)$, we get an Euclidean approximation 
of what the divergence measures, which we call the ``natural decrement''
as a reference to the Newton decrement used 
in the affine-invariant analysis of Newton's method \citep{nesterov1994interior}	
\aligns{
	\text{natural decrement: }
	&&
	\frac{1}{2}\norm{\nablaLoss(\np)}_{I_{\data,\lat}(\np)^{-1}}^2
	&&&&
	\text{Newton decrement: }
	&&
	\frac{1}{2}\norm{\nablaLoss(\np)}_{\nabla^2\Loss(\np)^{-1}}^2
}
The invariance to homeomorphisms can be shown as follow. 
Consider an alternative parametrization of the objective, 
$\Loss_{\text{alt}}(\psi) = \Loss(f(\psi))$
where $(f, f^{-1})$ is the mapping between the parametrizations, 
$\np = f(\psi)$ and $\psi = f^{-1}(\np)$.
We use $I_{\data,\lat\cond\np}$ and $I_{\data,\lat\cond\psi}$ to differentiate between the 
\aFIM of the two parametrizations. We have
\aligns{
	\nabla \Loss_{\text{alt}}(\psi)
	= \nabla \Loss(f(\psi))
	= \Jac f(\psi) \nablaLoss(\np)
	&&\text{ and }&&
	I_{\data,\lat\cond\psi}(\psi) = \Jac f(\psi)^\top \, I_{\data,\lat\cond\np}(\np) \, \Jac f(\psi),
}
where the second equality is a property of the Fisher information, 
shown in \cref{app:recap-fisher}.
The two parametrizations then give the same natural decrement,
\aligns{
	\norm{\nablaLoss_{\text{alt}}(\psi)}_{I_{\data,\lat\cond\psi}(\psi)^{-1}}^2
	&= 
	\lin{
		\nablaLoss_{\text{alt}}(\psi), 
		I_{\data,\lat\cond\psi}(\psi)^{-1}
		\nablaLoss_{\text{alt}}(\psi)
	}
	\\
	&=
	\lin{
		\Jac f(\psi) \nablaLoss(\np), 
		\Jac f(\psi)^{-1} \, I_{\data,\lat\cond\np}(\np)^{-1} \, \Jac f(\psi)^{-\top}
		\Jac f(\psi) \nablaLoss(\np)
	}
	\\
	&=
	\lin{
		\nablaLoss(\np), 
		I_{\data,\lat\cond\np}(\np)^{-1}
		\nablaLoss(\np)
	}
	\quad = 
	\norm{\nablaLoss(\np)}_{I_{\data,\lat\cond\np}(\np)^{-1}}^2.
}

\subsection{Generalized \aEM schemes (\cref{thm:multiplicative-error,thm:additive-error})}
\label{app:subsec-convergence-gem}
\label{sec:gem}

\begin{thmbox}
\restateMultiplicativeError*
\end{thmbox}

\begin{proof}[Proof of \cref{thm:multiplicative-error}]
	Recall the definition of the multiplicative error in \ref{ass:multiplicative-error},
\aligns{
	\Expect{\Q{\npt}{\nptt} - \Q{\npt}{\npt^*} \cond \npt}
	\leq 
	(1-c) \paren{\Q{\npt}{\npt} - \Q{\npt}{\npt^*}}.
}
By adding $\Q{\npt}{\npt^*} - \Q{\npt}{\npt}$ to both sides,
we get the following guarantee,
\aligns{
	\Expect{\Q{\npt}{\nptt} - \Q{\npt}{\npt} \cond \npt}
	\leq 
	- c \paren{\Q{\npt}{\npt} - \Q{\npt}{\npt^*}}.
}
Plugging this inequality in the decomposition of the 
objective function \peqref{eq:em-as-loss-and-entropy},
\aligns{\ts
	\Expect{\Loss(\nptt) - \Loss(\npt) \cond \npt}
	= \Expect*{\Q{\npt}{\nptt} - \Q{\npt}{\npt} + 
	\underbrace{\Hent{\npt}{\npt} - \Hent{\npt}{\nptt}}_{\leq0} \cond \npt}
	\leq
	- c \, \paren{\Q{\npt}{\npt} - \Q{\npt}{\npt^*}}.
}
Using the same development as 
in \cref{thm:convergence-em-as-md}, 
$\Q{\npt}{\npt} - \Q{\npt}{\npt^*} = D_{\AS}(s(\npt), \mp_t)$,
and reorganizing gives
\aligns{\begin{aligned}
	D_{\AS}(s(\npt), \mp_t)
	\leq 
	\frac{1}{c} \Expect{\Loss(\npt) - \Loss(\nptt) \cond \npt}.
\end{aligned}}
Taking full expectation, averaging over all iterations 
and bounding $\Expect{\Loss(\np_T)} > \Loss^*$ finishes the proof.
\end{proof}

\begin{thmbox}
\restateAdditiveError*
\end{thmbox}

\begin{proof}[Proof of \cref{thm:additive-error}]
Recall the definition of the additive error in \ref{ass:additive-error},
\aligns{
	\Expect{\Q{\npt}{\nptt} - \Q{\npt}{\npt^*} \cond \npt}
	\leq 
	\epsilon_t.
}
Plugging this inequality in the decomposition of the 
objective function \peqref{eq:em-as-loss-and-entropy},
\aligns{\ts
	\Expect{\Loss(\nptt) - \Loss(\npt) \cond \npt}
	&= \Expect*{\Q{\npt}{\nptt} - \Q{\npt}{\npt} + 
	\underbrace{\Hent{\npt}{\npt} - \Hent{\npt}{\nptt}}_{\leq 0} \cond \npt},
	\\[-1em]
	&\leq
	\Expect{\Q{\npt}{\nptt}	- \Q{\npt}{\npt} \cond \npt},
	\\
	&=
	\Expect{\Q{\npt}{\nptt}	- \Q{\npt}{\npt^*} \cond \npt}
	+ \Q{\npt}{\npt^*} - \Q{\npt}{\npt}
	\leq
	\epsilon_t
	+ \Q{\npt}{\npt^*} - \Q{\npt}{\npt}.
}
Using the same developments as 
\cref{thm:convergence-em-as-md,cor:convergence-em-as-md}, 
we have 
$\Q{\npt}{\npt} - \Q{\npt}{\npt^*} = D_{\AS}(s(\npt), \mp_t)$,
and
\aligns{\begin{aligned}
	D_{\AS}(s(\npt), \mp_t)
	\leq 
	\Expect{\Loss(\npt) - \Loss(\nptt) \cond \npt} + \epsilon_t.
\end{aligned}}
Taking full expectations and averaging over all iterations 
and bounding $\Expect{\Loss(\np_T)} > \Loss^*$ finishes the proof.
\end{proof}

\subsection{Relative strong-convexity and the ratio of missing information}
\label{app:subsec-relative-sc}

\begin{thmbox}
\restateStrongConvexity*
\end{thmbox}
\begin{proof}[Proof of \cref{prop:sc}]
That the objective is $\alpha$-strong convexity relative to $\A$ 
is equivalent to
\aligns{
	\nabla^2 \Loss(\np) \succeq \alpha \nabla^2 \A(\np).
}
By the decomposition of the objective \peqref{eq:em-as-loss-and-entropy},
\aligns{
	\nabla^2 \Loss(\np) = \nabla^2 \Q{\np}{\np} - \nabla^2 \Hent{\np}{\np}.
}
If the complete-data distribution is in the exponential family, 
the Hessian of the surrogate is 
\aligns{
	\nabla^2\Q{\np}{\np} 
	&= \medint \mneg\nabla^2\log p(\data,\lat \cond\np) \, p(\lat\cond\data,\np) \dif{\lat}
	\\[-.25em]
	&= \medint \nabla^2\brackets{\A(\np) - \lin{\stats(\data,\lat), \np} } \, p(\lat\cond\data,\np) \dif{\lat}
	= \nabla^2\A(\np),
}
where $\nabla^2\A(\np)$ is the \aFIM of the complete-data distribution, 
$I_{\data,\lat}(\np)$.
The Hessian of the entropy term is the Fisher of the conditional distribution,
$I_{\lat\cond\data}(\np)$,
\aligns{
	\nabla^2\Hent{\np}{\np} = \medint \mneg\nabla^2 p(\lat\cond\data,\np) \, p(\lat\cond\data,\np) \dif{\lat}
	= I_{\lat\cond\data}(\np).
}
This gives that the relative $\alpha$-strong convexity of $\Loss$ is equivalent to 
\aligns{
	I_{\data,\lat}(\np) - I_{\lat\cond\data}(\np) \succeq \alpha I_{\data,\lat}(\np).
}
Multiplying by the inverse of $I_{\data,\lat}(\np)$, 
which always exist if the exponential family is minimal (\ref{ass:1}), 
\aligns{
	I - I_{\data,\lat}(\np)^{-1} I_{\lat\cond\data}(\np) \succeq \alpha I
	&&\iff&&
	(1-\alpha) I \succeq I_{\data,\lat}(\np)^{-1} I_{\lat\cond\data}(\np) = M(\np),
}
where $I$ is the identity matrix. 
This gives that $\Loss$ is $\alpha$-strongly convex relative to $\A$ 
on a subset $\Theta$ if and only if the largest eigenvalue of 
the missing information is bounded by $1-\alpha$.
\qedhere
\end{proof}

\clearpage
\subsection{Local convergence of \aEM (\cref{cor:convex,cor:convergence-sc})}
\label{app:subsec-convergence-em-convex}

We now present proofs for the locally convex 
and relatively strongly-convex settings
in \cref{cor:convex,cor:convergence-sc}, 
restated below for convenience.

\begin{thmbox}
\restateCorConvex*
\end{thmbox}

\begin{thmbox}
\convergenceEmAsMdStronglyConvex*
\end{thmbox}

Both corollaries are direct consequences of Theorem 3.1 
in \citet{lu2018relatively}
with $L = 1$ if initialized in a convex 
or relatively $(1-r)$-strongly convex region.
We present here an alternative proof.

\begin{theorem}[Simplified version of Theorem 3.1 \citep{lu2018relatively}]
Let \ref{ass:1}--\ref{ass:3} hold and 
$\Loss$ be a convex and $1$-smooth function relative to $\A$, with minimum at $\np^*$.
Mirror descent 
with step-size $\gamma=1$, leading to the update satisfying
$\dA(\nptt) = \dA(\npt) - \nablaLoss(\npt)$,
converges at the rate
\aligns{
	\Loss(\npT) - \Loss(\np^*) \leq \frac{1}{\maxiter}\D(\np^*, \npStart).
}
If, in addition, $\Loss$ is $\alpha$-strongly convex relative to $\A$, 
then
\aligns{
	\Loss(\npT) - \Loss(\np^*) \leq (1-\alpha)^\maxiter \D(\np^*, \npStart).
}
\end{theorem}
\begin{proof}
Recall that by definition of the update, 
$\dA(\nptt) = \dA(\npt) - \nablaLoss(\npt)$.
By relative smoothness, we have
\aligns{
	\Loss(\nptt) 
	\leq \Loss(\npt) \, +& \, \lin{\nabla \Loss(\npt),\nptt-\npt} + \D(\nptt,\npt).
\intertext{
{\bf We first show that the algorithm makes progress} at each step,
$\Loss(\nptt) \leq \Loss(\npt)$, by showing that 
}
	\Loss(\nptt) 
	- \Loss(\npt)
	\leq \, & \, \lin{\nabla \Loss(\npt),\nptt-\npt} + \D(\nptt,\npt)
	\leq - \D(\npt,\nptt).
}
Substituting the gradient by $\dA(\npt) - \dA(\nptt)$ we have that
\aligns{
	\lin{\nabla \Loss(\npt),\nptt-\npt} + \D(\nptt,\npt)
	&=
	\lin{\dA(\npt) - \dA(\nptt),\nptt-\npt} + \D(\nptt,\npt).
\intertext{Expanding the Bregman divergence
as $\D(\nptt,\npt) = \A(\nptt) - \A(\npt) - \lin{\dA(\npt),\nptt-\npt}$,
we get the simplification}
	&=
	\lin{\dA(\npt) - \dA(\nptt),\nptt-\npt} + \A(\nptt) - \A(\npt) - \lin{\dA(\npt), \nptt - \npt},
	\\
	&=
	- \lin{\dA(\nptt),\nptt-\npt} + \A(\nptt) - \A(\npt)
	\\
	&= - \D(\npt, \nptt) \leq 0.
}

{\bf We now relate the progress to the Bregman divergence to the minimum.}
We will show that
\aligns{
	\Loss(\nptt) - \Loss(\np^*) \leq \D(\np^*,\npt) - \D(\np^*,\nptt).
}
Starting from relative smoothness, 
\aligns{
	\Loss(\nptt) 
	&\leq \Loss(\npt) + \lin{\nabla \Loss(\npt),\nptt-\npt} + \D(\nptt,\npt),
	\\
	&= \Loss(\npt) + \lin{\nabla \Loss(\npt),\nptt-\np^*+\np^*-\npt} + \D(\nptt,\npt),
	\tag{$\pm \lin{\nabla \Loss(\npt),\np^*}$}
	\\
	&= \Loss(\npt) + \lin{\nabla \Loss(\npt),\np^*-\npt}
	+ \lin{\nabla \Loss(\npt),\nptt-\np^*}
	+ \D(\nptt,\npt).
\intertext{
	By convexity, we have that $\Loss(\np^*) \geq \Loss(\npt) + \lin{\nabla \Loss(\npt),\np^*-\npt}$
	and
}
	&\leq \Loss(\np^*)
	+ \lin{\nabla \Loss(\npt),\nptt-\np^*}
	+ \D(\nptt,\npt).
\intertext{
	Using that the update satisfies $\dA(\nptt) = \dA(\npt) - \nablaLoss(\npt)$, 
	we can rewrite the gradient as 
}
	&= \Loss(\np^*)
	+ \lin{\dA(\npt) - \dA(\nptt),\nptt-\np^*}
	+ \D(\nptt,\npt),
	\\
	&= \Loss(\np^*)
	+ \lin{\np^*-\nptt,\dA(\nptt)-\dA(\npt)}
	+ \D(\nptt,\npt).
\intertext{
Using the three point property, 
$\D(\np^*, \npt) = \D(\np^*,\nptt) + \lin{\np^*-\nptt, \dA(\nptt) - \dA(\npt)} + \D(\nptt,\npt)$
and
}
	&= \Loss(\np^*)
	+ \D(\np^*,\npt) - \D(\np^*,\nptt).
}
Reorganizing the terms yields the inequality 
$\Loss(\nptt) - \Loss(\np^*) \leq \D(\np^*,\npt) - \D(\np^*,\nptt)$.

Using that the algorithm makes progress and 
summing all iterations yields 
\aligns{
	T\paren{\Loss(\npT) - \Loss(\np^*)}
	\leq \fullsum \Loss(\nptt) - \Loss(\np^*) 
	\leq \fullsum \D(\np^*,\npt) - \D(\np^*,\nptt)
	\leq \D(\np^*, \npStart).
}
Dividing by $T$ finishes the proof for the convex case.

~

{\bf For the relatively strongly-convex case}, 
we will show that the Bregman divergence also converges linearly,
\aligns{
	\D(\np^*, \nptt) \leq (1-\alpha)^t \D(\np^*, \npStart).
}
Combining this contraction with earlier result that 
$\Loss(\nptt) - \Loss(\np^*) \leq \D(\np^*,\npt) - \D(\np^*,\nptt)$
implies 
\aligns{
	\Loss(\nptt) -\Loss(\np^*)
	\leq \D(\np^*,\npt) - \D(\np^*,\nptt)
	\leq \D(\np^*, \npt) \leq (1-\alpha)^t \D(\np^*, \npStart).
}

~

In addition to the three point property, 
we will use the following two results to show the linear rate 
of convergence. 
By the relative $\alpha$-strong convexity of $\Loss$, 
\aligns{
	\Loss(\np^*) \geq \Loss(\npt) + \lin{\nablaLoss(\npt), \np^* - \npt} + \alpha \D(\np^*, \npt)
	&\implies&
	\lin{\nablaLoss(\npt), \np^* - \npt} \leq \Loss(\np^*) - \Loss(\npt) - \alpha\D(\np^*,\npt),
	\tag{A}
	\label{eq:app-proof-rsc}
\intertext{
And by the first result we showed, the algorithm 
makes progress proportional to $\D(\npt, \nptt)$,
}
	\Loss(\nptt) - \Loss(\npt) \leq - \D(\npt, \nptt)
	&\implies&
	\D(\npt, \nptt) \leq \Loss(\npt) - \Loss(\nptt)
	\tag{B}
	\label{eq:app-proof-progress}
}

~

Using the three point property, we can expand the divergence as
\aligns{
	\D(\np^*, \nptt)
	&= 
	\D(\np^*, \npt) + \lin{\np^* - \npt, \dA(\npt) - \dA(\nptt)} + \D(\npt, \nptt).
	\\
\intertext{
	Replacing $\dA(\npt) - \dA(\nptt)$ by the gradient at $\npt$, we have
}
	&=
	\D(\np^*, \npt) + \lin{\nablaLoss(\npt), \np^* - \npt} + \D(\npt, \nptt).
\intertext{
	Using the relative $\alpha$-strong convexity of $\Loss$ (\ref{eq:app-proof-rsc}),
}
	&\leq (1-\alpha) \D(\np^*, \npt) + \paren{\Loss(\np^*) - \Loss(\npt)} + \D(\npt, \nptt).
\intertext{
	Using the progress bound (\ref{eq:app-proof-progress}),
}
	&\leq (1-\alpha) \D(\np^*, \npt) + \paren{\Loss(\np^*) - \Loss(\npt)} + 
	\paren{\Loss(\npt) - \Loss(\nptt)},
	\\
	&= (1-\alpha) \D(\np^*, \npt) + \paren{\Loss(\np^*) - \Loss(\nptt)}.
}
As $\Loss(\np^*) \leq \Loss(\nptt)$, 
we get that 
$\D(\np^*, \nptt)\leq (1-\alpha) \D(\np^*, \npt)$.
Recursing finishes the proof,
\aligns{
	\D(\np^*, \nptt) \leq (1-\alpha)^t \D(\np^*, \npStart).
	\tag*{\qedhere}
}
\end{proof}

\newpage
\section{Supplementary material for \cref{sec:em-general}:\newline
\nameref{sec:em-general}}
\label{app:estep-analysis}
\label{sec:e-step}

This sections extends the results on stationarity 
in \cref{sec:convergence-results} 
to handle cases where the objective and the surrogate 
can be non-differentiable. 
A simple example of this setting is 
a mixture of Laplace distributions.
It is still possible to optimize the \Mstep, 
but the theory does not apply as the Laplace is not 
in the exponential family.
The main problem for the analysis of
non-differentiable, non-convex objectives 
is that the progress at each step need not be related 
to the gradient (if it is even defined at the current point).
Asymptotic convergence can still be shown \citep{chretien2000kullback,tseng2004analysis},
but non-asymptotic results are not available
without stronger assumptions, such as the 
Kurdyka-\L{}ojasiewicz inequality or weak convexity.

Instead of focusing on the progress of the \Mstep,
we look here at the progress of the \Estep
under the assumption that the conditional distribution 
over the latent variables $p(\lat\cond\data,\np)$ 
is in the exponential family. 
This is a strictly weaker assumption, 
as it is implied if the complete-data distribution $p(\data,\lat\cond\np)$
is an exponential family distribution, 
but holds more generally. 
For example, it is satisfied by any finite mixture,
even if the mixture components are non-differentiable,
as for the mixture of Laplace distributions.
As a tradeoff, however, the resulting convergence results only 
describe the stationarity of the parameters controlling the latent variables.

To analyse the \Estep, 
we use the formulation of \aEM 
as a block-coordinate optimization problem.
Let $q(\lat \cond \phi)$ be an exponential family distribution 
in the same family as $p(\lat \cond \data, \np)$,
such that $\min_\phi \KL{q(z\cond\phi)}{p(\lat\cond \data,\np)} = 0$.
We can write the \Estep and \Mstep as an alternating optimization procedure on the augmented objective $\Aug$,
\aligns{
	\Aug(\np, \phi) = 
	-\medint \log\paren{
		\frac{p(\data,\lat\cond \np)}{q(z\cond\phi)}
	}\, q(\lat \cond \phi) \dif{\lat}
	&&\text{ such that }&&
	\Loss(\np) = \min_\phi \Aug(\np, \phi),
}
The parameters $\phi$ and $\np$ need not be defined on the same space, 
as $\phi$ only controls the conditional distribution over the latent variables 
and $\np$ controls the complete-data distribution.
The \aE and \aM steps then correspond to 
\aligns{
	\text{\aE-step:}&&\ts
	\phi_{t+1} = \arg\min_\phi \Aug(\np_t, \phi),
	&&
	\text{\aM-step:}&&\ts
	\np_{t+1} = \arg\min_\np \Aug(\np, \phi_{t+1}).
}
Two gradients now describe stationarity;
the gradient of the \Mstep,
$\nabla_{\np} \Aug(\npt, \phi_{t+1})$, 
which we studied before,
and the gradient of the \Estep,
$\nabla_\phi \Aug(\np_t, \phi_t)$.
Let $\stats$ and $\A$ be the sufficient statistics and log-partition function of $q(\lat\cond\phi)$,
and the natural and equivalent mean parameters be denoted by $(\phi_t, \mp_t)$.
We show the following, 
which is the analog of \cref{cor:convergence-em-as-md} 
for the conditional distribution over the latent variables $q(\lat\cond\phi)$.
\begin{restatable}{theorem}{restateTheoremEStep}
\label{thm:e-step}
Let Assumption \ref{ass:2} and \ref{ass:3} hold, 
and let $\np$ be the parameters of the complete-data distribution $p(\data,\lat\cond\np)$.
If the conditional distribution over the latent variables $q(\lat\cond\phi)$
is a minimal exponential family distribution 
with natural and mean parameters $(\phi, \mp)$,
\aligns{\SwapAboveDisplaySkip
	\fullmin \D(\phitt, \phit)
	\leq 
	\frac{\Loss(\npStart) - \Loss^*}{\maxiter}.
}
This implies convergence of the gradient in KL divergence 
as the (natural) gradient 
is $\nabla_{\mp} \Aug(\npt, \phi_t) = \phi_t - \phi_{t+1}$.
\end{restatable}
\begin{proof}[Proof of \cref{thm:e-step}]
Let us start by bounding the progress on the overall objective 
by the progress of the \aE-step;
\aligns{
	\Loss(\npt) - \Loss(\nptm)
	&=
	\overbrace{\Aug(\npt, \phitt) - \Aug(\nptm, \phit)}^{\text{Progress of the joint \aEM step}},
	\\
	&=
	\underbrace{\Aug(\npt, \phitt) 
	- \Aug(\npt, \phit)}_{\text{Progress of the \aE-step}}
	+ \underbrace{\Aug(\npt, \phit)
	- \Aug(\nptm, \phit)}_{\text{Progress of the \aM-step}}
	\leq
		\Aug(\npt, \phitt) - \Aug(\npt, \phit).
}
The last inequality holds as the \aM-step is guarantee to make progress.
To show that the progress of the \aE-step is the KL divergence
between $q(z\cond\phi)$ and $p(\lat\cond \data, \npt)$
we use the following substitution,
\aligns{
	\Aug(\np, \phi)
	=
	- \medint \log \frac{p(\data,\lat\cond \np)}{q(\lat \cond \phi)} \, q(\lat \cond \phi) \dif{\lat}
	&=
	- \medint \log \frac{p(\lat \cond \data, \np)}{q(\lat \cond \phi)} \, q(\lat \cond \phi) \dif{\lat}
	- \medint\log p(x \cond \np) q(z\cond\phi)\dif{\lat}
	\\
	&= \KL{q(z\cond\phi)}{p(\lat \cond \data, \np)} - \log p(x\cond\np).
}
Plugging the substitution in $\Aug(\npt, \phitt) - \Aug(\npt, \phit)$ yields
\aligns{
	\Aug(\npt, \phitt) 
	- \Aug(\npt, \phit)
	&= 
	\underbrace{\KL{q(z\cond\phitt)}{p(\lat \cond \data, \npt)}}_{=0}
	- \KL{q(z\cond\phit)}{p(\lat \cond \data, \npt)},
}
where the first term is 0 if $p(\lat\cond\data,\np)$ and $q(\lat\cond\phi)$ 
are in the same exponential family
and $\phitt$ is the exact solution%
\aligns{\ts
	s(\npt) = \Expect[p(\lat\cond \data, \npt)]{\stats(z)}, 
	&&
	\phitt = \dAS(s(\npt)).
}
Combining the bounds so far, we have that 
\aligns{
	\Loss(\npt) - \Loss(\nptm)
	&\leq 
	\Aug(\npt, \phitt) - \Aug(\npt, \phit)
	\leq 
	- \KL{q(z\cond\phit)}{p(\lat \cond \data, \npt)}.
}
To relate the progress to the gradient,
we express the KL divergence 
as a Bregman divergence in mean parameters, 
\aligns{
	\KL{q(z\cond\phit)}{p(\lat \cond \data, \npt)}
	= \DS(\mpt, \s(\npt)).
}
The gradient with respect to the mean parameters at $(\phit, \mpt)$ is then
\aligns{
	\nabla_{\mp} \Aug(\npt, \dAS(\mp)) \cond_{\mp=\mpt}
	&=
	\nabla_\mp 
		\KL{q(z\cond\dAS(\mp))}{p(\lat \cond \data, \npt)}
	\cond_{\mp=\mpt},
	\\
	&=
	\nabla_\mp D_{\AS}(\mp, \s(\npt)) \cond_{\mp=\mpt},
	\\
	&= 
	\nabla_{\mp} \brackets{\,
			\AS(\mp) - \AS(s(\npt)) - \lin{\dAS(s(\npt)), \mp - \s(\npt)}
	\,} \cond_{\mp=\mpt},
	\\
	&=
	\dAS(\mpt)
	- 
	\dAS(s(\npt))
	= \phit - \phitt.
}
We can then express the update of the \Estep from $\phit$ to $\phitt$ 
as a mirror descent step, 
updating the natural parameters using the gradient with respect to the natural parameters,
\aligns{
	\phitt
	&= \phit - \nabla_{\mp} \Aug(\npt, \dAS(\mpt)).
}
To express this update in natural parameters only, 
recall from \cref{app:recap-ef,app:recap-fisher}
that $\nabla^2 \AS(\mpt) = [\nabla^2 \A(\phit)]^{-1}$
and that $\nabla^2 \A(\phit)$ is the Fisher information matrix 
of the distribution $q(\lat\cond\phi)$, $I_{\lat\cond\phi}(\phit)$.
The update is then equivalent to a natural gradient update in natural parameters, 
as 
\aligns{
	\nabla_{\mp} \Aug(\npt, \dAS(\mpt))
	= \nabla^2 \AS(\mpt) \nabla_{\phi} \Aug(\npt, \phit)
	= [I_{\lat\cond\phi}(\phit)]^{-1} \nabla_{\phi} \Aug(\npt, \phit).
}
Using those expression for the \aKL divergence and parameter updates yields the bound
\aligns{
	\Loss(\npt) - \Loss(\nptm)
	\leq 
	- \KL{q(z\cond\phit)}{p(\lat \cond \data, \npt)}
	= - \D(\phit - \nabla_{\mp} \Aug(\nptt, \dAS(\mpt)), \phit).
}
Reorganizing terms gives
\aligns{
	\D(\phit - \nabla_{\mp} \Aug(\npt, \dAS(\mpt)), \phit)
	\leq 
	\Loss(\nptm) - \Loss(\npt).
}
And averaging over all iterations 
and bounding $\Loss(\np_\maxiter) > \Loss^*$ finishes the proof.
\end{proof}

\end{document}